\documentclass[letterpaper]{article} 
\usepackage{aaai2027}                
\nocopyright                         
\usepackage[hyphens]{url}            
\usepackage{graphicx}                
\usepackage{natbib}                  
\usepackage{caption}
\DeclareCaptionStyle{ruled}{labelfont=normalfont,labelsep=colon,strut=off} 
\usepackage{amsmath,amssymb,amsthm,booktabs,array}
\usepackage{algorithm,algpseudocode}

\usepackage{cleveref}

\graphicspath{{figs/}{./figs/}}

\newtheorem{theorem}{Theorem}
\newtheorem{proposition}{Proposition}
\newtheorem{corollary}{Corollary}
\theoremstyle{definition}
\newtheorem{assumption}{Assumption}

\theoremstyle{remark}
\newtheorem{remark}{Remark}
\crefname{algorithm}{Algorithm}{Algorithms}
\Crefname{algorithm}{Algorithm}{Algorithms}
\crefname{assumption}{Assumption}{Assumptions}
\Crefname{assumption}{Assumption}{Assumptions}

\newcommand{\Lset}{\ensuremath{\mathcal{L}}}      
\newcommand{\Cset}{\ensuremath{\mathcal{C}}}      
\newcommand{\Sset}{\ensuremath{S}}                
\newcommand{\Sbar}{\ensuremath{\bar{S}}}
\newcommand{\rcls}{\ensuremath{g_{\theta}}}       
\newcommand{\gen}{\ensuremath{q_{\phi}}}          
\newcommand{\yhat}{\ensuremath{\hat{y}}}
\newcommand{\phat}{\ensuremath{\hat{p}}}
\newcommand{\acoal}{\ensuremath{A}}
\DeclareMathOperator*{\argmax}{arg\,max}
\newcommand{\amin}{a_{\min}}

\newcommand{\rhocpl}{\rho_{\mathrm{cpl}}}

\title{Conditional Validity for Adaptive Modality Acquisition:\\When the Policy Chooses Its Own Calibration Group}
\author{Melika Baghi}
\affiliations{Georgia Institute of Technology \\ Atlanta, GA, USA}

\begin{document}
\maketitle

\begin{abstract}
A multimodal system may begin inference holding only some of its inputs and may
acquire the rest at a cost. With adaptive acquisition, the policy determines which inputs are ultimately observed, so we state risk control conditional on that terminal input pattern. Conditional calibration normally assumes the grouping map is fixed independently of the calibration sample, which policy-induced grouping does not satisfy. We characterize when pattern-conditional risk control remains valid and give two finite-sample constructions: threshold-free routing with calibration applied at the
terminal pattern, and simultaneous validation of complete policy--pattern pairs,
which lets calibration data select the deployed policy. A counterexample shows that validity proved for a calibration-independent grouping map need not transfer once the policy makes the terminal group calibration-dependent. We call the resulting method RouteCert. On a clinical electrocardiogram task with a staged, cost-ordered lead protocol, the deployed policy answers $71.2\%$ of held-out patients at an observed $7.4\%$ disagreement with
the cardiologist's diagnosis at $48.8\%$ of the prespecified ordinal cost of acquiring every
stage, and all three acquisition stages are validated separately. On masked multimodal benchmarks, validating pointwise at each terminal pattern holds observed worst-pattern selective risk, measured against the full-information reference decision rather than the true label, at $0.034$ where a pooled design reaches $0.145$ against a $0.10$ cap, at a comparable answered fraction ($0.350$ vs $0.342$); under the budget-matched simultaneous comparison the answered fraction falls to $0.305$.
\end{abstract}

\section{Introduction}
\label{sec:intro}

Multimodal systems are increasingly deployed where the inputs arrive incomplete and
further inputs can be obtained at a price. Inference-time modality selection for
incomplete multimodal classification is an active topic \citep{du2026dymo}, active
feature acquisition has been surveyed as a field in its own right
\citep{aronsson2025afa}, and cohort-level acquisition has been studied for settings
where the acquisition decision itself carries cost \citep{rheude2025cama}. In each
case a system holds some sources, a sensor channel, an assay, a camera view, and
may buy others for time, money or compute. A body-worn sensor rig can power up a
second accelerometer, a molecular-typing pipeline can run one more omics assay, and
a retrospective study can pull a scan that was archived but never read. The
decision is made per input: whether the evidence in hand suffices to commit, or the
system should pay for more.

Three lines of work bear on this and answer different questions. Missing-modality
methods return a point prediction \citep{du2024tip,liang2023pid}, dynamic acquisition
buys whichever source most raises expected accuracy per unit cost
\citep{du2026dymo,ma2019eddi}, and conformal and selective prediction calibrate a
fixed predictor over a grouping rule fixed independently of the calibration data
\citep{angelopoulos2024crc,fan2025maskcp}. The combination we study is certification
conditional on the terminal pattern that the selected acquisition policy itself
produces, which \cref{sec:related} places against each.

The question this paper answers is how a system can retain finite-sample guarantees
conditional on its terminal acquisition pattern when that pattern is generated by an
adaptive policy. The obstacle is an assumption mismatch rather than an error in any
existing argument. Split conformal prediction calibrates a cutoff $\tau$ so that a
future exchangeable case scores at or below it with probability $1-\alpha$, and the
grouping it conditions on is fixed in advance. Let acquisition instead stop as soon
as a case's score clears $\tau$: the group a case lands in is then decided by the
same number the group is graded against, and the fixed-map argument no longer
applies to it. \Cref{rem:sharp} makes this precise with a construction on which
direct transfer gives nothing, which is why an additional construction is needed
rather than a tighter constant. \Cref{fig:routing} shows the boundary and the two
certifications that respect it.

We develop those certifications and instantiate them. We call the resulting
acquisition-and-certification method RouteCert. A loop draws plausible completions of
the missing sources, reads from a reference model the set of answers still in play,
and while more than one answer survives buys the source or small coalition with the
best expected shrinkage per unit cost, never treating completions as acquired
evidence. Once it stops, a single guarantee is applied at whatever pattern the case
reached. The object certified is the complete policy together with the pattern it
produces, which is what licenses calibration-dependent selection. Our experiments cover four real datasets, three multimodal benchmarks under controlled masking, the protocol under which calibration can replay each policy's purchases, and a staged clinical electrocardiogram protocol.

\begin{itemize}\itemsep0pt \parskip0pt \topsep2pt
\item \textbf{Validity boundary and certification framework.} We formalise
pattern-conditional certification for adaptive modality acquisition, identify when
fixed-group calibration arguments apply (\cref{asm:polcal,rem:sharp}), and give two
finite-sample constructions (\cref{thm:coverage,prop:rcps}). Certifying complete
policy--pattern pairs simultaneously then lets the deployed policy be selected on the
same calibration data while every answered pattern keeps its cap, at a stated
finite-sample cost (\cref{thm:aware,prop:cost}).
\item \textbf{Acquisition instantiation.} We introduce RouteCert, a completion-based,
cost-aware coalition acquisition rule that operates inside these conditions
(\cref{alg:routecertmain}).
\item \textbf{Empirical evaluation.} We measure terminal-pattern risk control
against pooled recalibration, the cost of calibration-aware routing, when coalitions
help, and how fragmentation bounds conditional resolution (\cref{sec:exp}).
\end{itemize}

\section{Related Work}
\label{sec:related}

Two questions separate the closest work: whether the deployed policy or threshold
may be selected using calibration data, and whether the guarantee is conditioned on
a group that the selected policy creates. Missing-modality prediction and dynamic acquisition answer
neither. Imputation and
robust-representation methods return a point prediction
\citep{du2024tip,palumbo2024cmvae,tsai2019multimodal,arevalo2017gmu,liang2023pid},
and dynamic modality selection and active feature acquisition choose whichever
source most raises expected accuracy or information per unit cost
\citep{du2026dymo,he2024efficient,panda2021adamml,ma2019eddi,li2021gsmrl,covert2023dime,aronsson2025afa,rheude2025cama},
without attaching a finite-sample guarantee to the resulting decisions.

Conformal and selective prediction supply guarantees but fix the grouping map.
Group-conditional, shift-weighted and risk-controlling extensions
\citep{vovk2005conformal,angelopoulos2021conformal,romano2020classification,sadinle2019least,vovk2003mondrian,gibbs2025conditional,tibshirani2019weighted,fannjiang2022feedback,bates2021rcps,angelopoulos2021ltt,angelopoulos2024crc},
abstention methods
\citep{chow1970optimum,geifman2019selectivenet,garciagalindo2024,scrc2025} and
mask-conditional methods \citep{zaffran2023missing,fan2025maskcp,any2any,chandy2026shapley}
allow the group to depend on the input but require the map assigning it to be fixed independently of calibration; post-selection guarantees \citep{jin2023selection} take the selection rule as given. \Citet{hoarau2026alma} combine active learning with fixed-order test-time acquisition, stopping when a modality-specific conformal set becomes a singleton, with a marginal statement per predictor. We instead study guarantees conditional on the terminal pattern an adaptive policy produces, including when that policy is selected using calibration data.

Concurrent work relaxes the other axis. \Citet{prinster2026cpc} and
\citet{yu2026adaptivecrc} allow calibration-dependent selection with the evaluation
group fixed, and BCEA \citep{xu2026bcea} incorporates a fixed visual-evidence acquisition policy
into the score and recalibrates globally on post-acquisition outcomes. These
directions are complementary to ours: they hold the evaluation group fixed and move
the selection, whereas we certify conditional on the terminal pattern that the
selected policy produces, treating that pattern as part of the certified object. Our pooled recalibration arm
transfers BCEA's calibration design to our setting and is not a reproduction of
BCEA on its native claim-grounding task (\cref{tab:pervsglobal}).

\section{Problem Formulation}
\label{sec:problem}

A case carries $M$ sources; a benchmark's sensor streams or feature blocks each
count as one. Its \emph{pattern} is the subset it observes,
$\Sset\subseteq\{1,\dots,M\}$, with missing complement $\Sbar$ and observed blocks
$x_\Sset$. Missing source $m$ costs $c_m$ under a budget $B$, and a group
$\acoal\subseteq\Sbar$ costs $c(\acoal)=\sum_{m\in\acoal}c_m$. An \emph{acquisition
policy} $d$ maps a case to a purchase path: from the initial pattern it repeatedly
buys an affordable group and updates $\Sset$, or stops. Writing
$\Sset_{\mathrm{fin}}(d,x)$ for the pattern where it stops, $d$ then answers with a
label or declines, so a policy is a complete decision rule: what to buy, when to
stop, and whether to answer.

Two targets are used. The default is the \emph{reference answer}
$y^{\mathrm{full}}=\argmax_y \rcls(x)_y$, the label a model trained on complete
inputs would give if every source were present, which measures whether a
partial-information answer is stable; where calibration examples carry ground truth
we target the \emph{true label} instead. For a policy $d$ and pattern $\Sset$, the
answer rate $a_\Sset(d)$ is the fraction of cases reaching $\Sset$ that $d$ answers
and the \emph{selective risk} $R_\Sset(d)$ is the probability that an answered case
there is wrong against the chosen target. We want a pattern-conditional cap,
$R_\Sset(d)\le\alpha$ at every answered pattern, holding with probability at least
$1-\delta$ over a calibration sample of $n$ cases from the same distribution.

We distinguish policies by whether routing may use calibration-derived quantities. A policy is \emph{calibration-blind} if its
purchases and stopping depend only on the input and the models, so
$\Sset_{\mathrm{fin}}$ is a fixed function of the input, and
\emph{calibration-aware} otherwise. The distinction matters because the conditioning
group is an output of the policy: if a cutoff $\tau$ estimated from the calibration
sample also decides where acquisition stops, changing $\tau$ changes which cases
land in $\Sset$, so the population against which $\tau$ is evaluated moves with
$\tau$ itself and the exchangeability step of the split-conformal argument no longer
applies. We therefore fix, before seeing the certification data, a
\emph{pre-registered family} $\mathcal{D}$ of complete policies with the set
$\mathcal{S}$ of patterns it can produce, and certify pairs $(d,\Sset)$ jointly.
Selection of the deployed member from $\mathcal{D}$ may then use that data.

\section{Methodology}
\label{sec:method}

RouteCert has two phases: an acquisition loop, then a single guarantee. \Cref{alg:routecertmain} states it precisely. The residual set only routes, and Phase~2 alone carries the guarantee (\cref{tab:objects}).

\paragraph{Models and the residual answer set.}
Two models, both fit on data disjoint from calibration, do the work: a completion model $\gen(z_{\Sbar}\mid x_\Sset)$ that samples plausible values for the missing blocks, and the reference model \rcls\ of \cref{sec:problem}. Drawing $K$ completions $z^{(k)}\sim\gen(\cdot\mid x_\Sset)$ gives the \emph{residual label set}
\[
  \Lset(x_\Sset)=\bigl\{\argmax_y \rcls(x_\Sset,z^{(k)}_{\Sbar})_y : k\le K\bigr\},
\]
the answers still in play once the missing sources are accounted for. Averaging over completions gives $\phat_y(x_\Sset)=\frac{1}{K}\sum_k\rcls(x_\Sset,z^{(k)}_{\Sbar})_y$, whose top class $\yhat$ is the emitted answer, with nonconformity score $s(x_\Sset,y)=1-\phat_y(x_\Sset)$. Phase~2 comes in two forms, each fit per pattern and neither consulted before the
loop stops. \textbf{RouteCert-Risk} is the deployed default: it answers when
$\phat_{\yhat}(x_\Sset)$ clears a validated cutoff $\lambda_\Sset$.
\textbf{RouteCert-Cov} forms $\Cset_\Sset(x_\Sset)=\{y:s(x_\Sset,y)\le\tau_\Sset\}$ at
the conformal quantile $\tau_\Sset$ and answers when that set is a singleton. All headline selective-risk results use RouteCert-Risk, whose cutoff is selected and validated on disjoint calibration subsets (\cref{prop:rcps}); \cref{thm:coverage} instead gives pattern-conditional coverage for the set-valued RouteCert-Cov. A \emph{coalition} is a group of sources bought together, scored by a look-ahead over groups of size at most $r_{\max}$.

While \Lset\ holds more than one answer, RouteCert scores small coalitions by expected reduction of \Lset\ per unit cost, buys the best, and repeats, stopping when \Lset\ collapses, no affordable group shrinks it, or the budget is spent. A candidate's block is not available before it is bought, so its gain $\widehat{\Delta}(\acoal)$ is estimated by averaging $\bigl(|\Lset(x_\Sset)|-|\Lset(x_\Sset,\tilde{x}_\acoal)|\bigr)/c(\acoal)$ over $J$ hypothetical versions $\tilde{x}_\acoal\sim\gen(\cdot\mid x_\Sset)$ of the group, each inner residual set recomputed from $K$ fresh completions, with the true block revealed only after purchase (\cref{alg:deploy}). With $r_{\max}{=}2$ each round scores the missing singles and, only when no single helps, the missing pairs, so the action count grows as $M+\binom{M}{2}$ rather than $2^M$ ($1.3$\,ms per case at $M{=}5$). This score is a design choice inside the freedom the theory leaves, not part of any guarantee. What matters for the guarantee is that everything in the loop depends only on the input, the models and the loop's own seeded randomness: no calibrated threshold is read. Only after it stops does Phase~2 consult calibration, at the terminal pattern, with parameters fit on calibration cases pushed through the same loop, so a case can stop the loop at $|\Lset|{=}1$ and still decline.

Acquired blocks are replaced by their observed values, so completions only
marginalize the sources still missing. Completion-model dependence and an optional
true-label screen are in \cref{app:methoddetails}; unless stated, results use the
reference-relative guarantee.        \begin{algorithm}[t]
\small
\caption{RouteCert at deployment: acquisition, then one guarantee.}
\label{alg:routecertmain}
\begin{algorithmic}[1]
\Require case $x$ with pattern $\Sset$, budget $B$, costs $c$, look-ahead order $r_{\max}$,
         completion model $\gen$, reference model $\rcls$, per-pattern cutoffs
         $\{\lambda_\Sset\}$ fitted on calibration
\State draw $K$ completions of the missing blocks and form $\Lset(x_\Sset)$
\While{$|\Lset(x_\Sset)|>1$ and some group is affordable}
  \State score each affordable single source by expected reduction of $|\Lset|$ per
         unit cost, estimated from completions
  \State \textbf{if} some single scores positive \textbf{then} buy the best single
  \State \textbf{else} score affordable groups of size up to $r_{\max}$, buy the best,
         and \textbf{break} if none scores positive
  \State reveal the purchased blocks; update $\Sset$, $B$, and $\Lset$
\EndWhile
\State $\Sset_{\mathrm{fin}}\gets\Sset$;\quad
       $\phat\gets$ completion-averaged posterior at $\Sset_{\mathrm{fin}}$
\If{$\max_y\phat_y \ge \lambda_{\Sset_{\mathrm{fin}}}$}
  \State \Return $\argmax_y\phat_y$
\Else
  \State \Return \textsc{decline}
\EndIf
\end{algorithmic}
\end{algorithm}

\section{Validity After Adaptive Acquisition}\label{sec:guarantee}

Throughout, $\alpha$ is the target risk or miscoverage level and $\delta$ the probability that a finite-sample guarantee fails, both $0.10$ unless noted. The decision is per input but the guarantee is not: risk is controlled over a population of future cases, not for any single prediction.

The theoretical contribution is \cref{thm:coverage,thm:aware} together with the
validated rule of \cref{prop:rcps}: three certified objects, differing in what the
guarantee attaches to and in whether routing may read a calibrated quantity. \Cref{thm:coverage} gives pattern-conditional coverage for the
set-valued RouteCert-Cov layer; \cref{prop:rcps} gives answered-case selective-risk
validation for the deployed RouteCert-Risk cutoff; \cref{thm:aware} certifies
complete policy--pattern pairs simultaneously. \Cref{rem:sharp} marks the boundary
these constructions are built to respect. \textbf{Regime~A, threshold-free routing:} if the policy never reads calibration
the terminal pattern is a fixed function of the input, so Mondrian split-conformal
prediction \citep{vovk2003mondrian} applies with the policy's own output as
taxonomy (\cref{asm:polcal,thm:coverage}). \textbf{Regime~B, calibration-aware policy-family certification:} \cref{thm:aware} certifies whole policies over a family fixed in advance, so the one deployed may be picked, and may route, on the certification data itself.

\Cref{rem:cost} and \cref{prop:cost} collect what either route costs, and proofs are in \cref{app:theory}.

\begin{assumption}[Policy-consistent calibration]\label{asm:polcal}
Let $\pi$ be the acquisition loop, mapping an input to its purchases and final
pattern $\Sset_{\mathrm{fin}}$ without reading any calibrated quantity. $\pi$ may
be randomized; its auxiliary randomness is drawn independently of calibration and
treated as part of the input, so exchangeability runs on the augmented pair. The
models and $\pi$ are fit on data disjoint from the calibration set, calibration
and test points are i.i.d., and every calibration point is routed by $\pi$ before
any pattern fits its own parameter (\cref{app:theory}).
\end{assumption}

\noindent
RouteCert's default loop stops on $|\Lset|$, reads no calibrated number, and so
complies; sample splitting extends the condition to a policy fitted on one fold and
graded on a disjoint one.

\begin{proposition}[Boundary of fixed-map conditional guarantees]\label{rem:sharp}
There are a score distribution and a routing rule, measurable in the input and the
calibration sample, for which the group-conditional conclusion holds at level
$1-\alpha$ for every grouping map fixed independently of that sample, while the
deployed rule attains pattern-conditional coverage $0$. Take i.i.d.\ $U[0,1]$
scores, let $\tau$ be the calibration quantile, route a case into $\Sset$ when its
score exceeds $\tau$, and form the answer set inside $\Sset$ by the same rule
$s\le\tau$. Every case routed to $\Sset$ is excluded from its own answer set. The
gap is therefore not a matter of constants: no bound obtained by holding the
grouping map fixed can be transported to a map the calibration sample defines.
\end{proposition}

\noindent
\textbf{Scope.} The proposition rules out unrestricted transfer
under arbitrary calibration dependence. It does not establish that
calibration-aware routing is unsafe in general: a particular policy may remain valid
under additional structure, stability, privacy or a bounded coupling, and
\cref{thm:aware} certifies such policies directly. The degradation is also
continuous rather than a knife edge: threshold reuse costs
$0.901{\pm}0.007\to0.888{\pm}0.007$ marginal coverage on CMU-MOSEI and decays to
$0.45$ under a synthetic sweep (\cref{fig:coupling}). Quantitative relaxations of
exchangeability bound this decay by a computable coupling term
\citep{barber2023beyond}; \cref{thm:aware} avoids paying it.

\noindent
\begin{figure*}[t]
\centering
\includegraphics[width=0.92\textwidth]{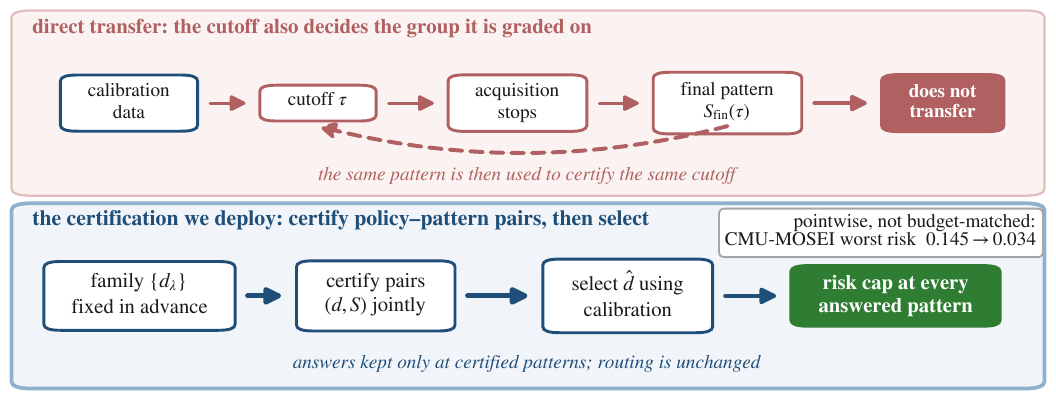}
\caption{\textbf{Validity boundary and two constructions for
policy-induced terminal patterns.} Above, the construction that does not transfer:
a cutoff $\tau$ estimated from calibration also decides where acquisition stops, so
it determines the terminal pattern $\Sset_{\mathrm{fin}}(\tau)$ that is then used to
certify $\tau$ itself. The dashed arrow marks that dependence, which
\cref{rem:sharp} shows is enough to void direct transfer of a fixed-map guarantee.
Below, the construction we develop and deploy: a family of complete policies $\{d_\lambda\}$ is fixed before certification, policy--pattern pairs $(d,\Sset)$ are certified jointly, a member $\hat d$ may then be selected on the same calibration data, and answers are kept only at certified patterns, leaving a selective-risk cap at every answered pattern (\cref{thm:aware}). At right, the measured effect of conditioning rather than recalibrating globally (\cref{tab:pervsglobal}).}
\label{fig:routing}
\end{figure*}

The deployed guarantee turns this condition into a two-split rule.

\begin{proposition}[RouteCert-Risk: the deployed rule]\label{prop:rcps}
Split each pattern's calibration cases in half. Pick a candidate confidence cutoff
on one half by any rule; on the other half test that single, now-fixed cutoff once,
with an exact Clopper--Pearson bound at confidence $1-\delta$ on the errors among
the cases it answers, and deploy it iff the bound is at most $\alpha$. Then with
probability at least $1-\delta$ over the calibration draw the pattern either
abstains or deploys a cutoff whose answered-case selective risk is at most $\alpha$,
using one exact test with no monotonicity assumption and no multiple-testing
correction.
\end{proposition}

\noindent
This is the rule behind every selective-risk number we report. A pattern answers
only with a cutoff that survived one exact test on data it was not chosen on, and
data-poor patterns abstain rather than guess (proof in \cref{app:theory}).

\begin{theorem}[RouteCert-Cov: per-pattern coverage]\label{thm:coverage}
Under \cref{asm:polcal}, for each terminal pattern $\Sset$ with
$\Pr[\Sset_{\mathrm{fin}}{=}\Sset]>0$ the answer set covers its target at
level $1-\alpha$: $\Pr[\,t\in\Cset_\Sset(x_\Sset)\mid\Sset\,]\ge1-\alpha$, the
probability taken over the joint draw of calibration set and test case. The
target $t$ is $y^{\mathrm{full}}$ by default and the true label under direct
ground-truth calibration, the argument being identical either way.
\end{theorem}

\noindent
Among future cases ending at the same pattern the answer set contains its
target in at least a $1-\alpha$ fraction, averaged over the calibration draw. What
\cref{asm:polcal} buys is the right to condition on that pattern at all.

Splitting buys calibration-dependent routing by spending data twice and still forbids
the policy from reading the guarantee that grades it. Both restrictions come from
attaching the guarantee to a \emph{threshold at a pattern}; attach it to a
\emph{policy} and they lift. Fix in advance any family $\mathcal{D}$ of complete acquisition policies: each $d\in\mathcal{D}$ says what to buy, when to stop and what to answer, so it induces a final-pattern map $x\mapsto\Sset_{\mathrm{fin}}(d,x)$ and, at each pattern, a selective risk $R_\Sset(d)$ among the cases it answers there. Let $\mathcal{S}$ be the patterns the family can produce, a subset of the $2^M$
possible ones. Nothing asks $\mathcal{D}$ to be indexed or parametric: the object
certified is a policy and the group conditioned on is that policy's own output. Our
experiments instantiate $\mathcal{D}$ as a grid of global confidence cutoffs and as
a grid of per-pattern cutoff \emph{vectors}.

Concretely, take three policies indexed by cutoffs $\lambda\in\{0.6,0.7,0.8\}$, each
a complete rule saying what to acquire, when to stop, where it answers and where it
declines. A case reaching confidence $0.72$ after one acquisition is answered by
$d_{0.6}$ and $d_{0.7}$ there, while $d_{0.8}$ keeps acquiring and stops elsewhere.
All three are fixed before the certification sample is seen; certification then
evaluates every relevant pair at once and the calibration data may pick one. That is
data-dependent selection among fixed policies, not construction of a new policy
after seeing the results.

\begin{theorem}[Calibration-aware acquisition]\label{thm:aware}
Let $\mathcal{D}$ and $\mathcal{S}$ be fixed independently of the certification
sample, and let $\widehat{\mathcal{W}}$ be a set of pairs
$(d,\Sset)\in\mathcal{D}\times\mathcal{S}$ selected from that sample by any
procedure controlling the family-wise error rate at level $\delta$ for the null
hypotheses $H_{d,\Sset}:R_\Sset(d)>\alpha$. Then
\[
\Pr\bigl[\,R_\Sset(d)\le\alpha\ \ \text{for every }(d,\Sset)\in\widehat{\mathcal{W}}\,\bigr]\ \ge\ 1-\delta .
\]
Consequently a practitioner may choose any $\hat d\in\mathcal{D}$ \emph{using that
same sample} and deploy it with its answers restricted to the patterns $\Sset$ for
which $(\hat d,\Sset)\in\widehat{\mathcal{W}}$, declining at the rest, and every
answered pattern still carries the $\alpha$ cap. What is relaxed is the
\emph{selection}: the deployed policy, including where it buys and where it stops,
may depend arbitrarily on the certification data, whereas \cref{asm:polcal}
requires the routing itself to be independent of it. What is not relaxed is the
family, which must be fixed independently of that sample; the theorem licenses
calibration-dependent choice from a calibration-independent family. The
restriction must suppress answers only, since altering where the policy stops
would alter which cases reach the later patterns, and their certificates are
computed under $\hat d$; a policy that continues instead is a different member of a
larger family and is certified as one (\cref{cor:continue}).
\end{theorem}

\noindent
Because the statement holds simultaneously over pairs, selecting the deployed policy using the same sample
does not invalidate the certified pairs.

\textbf{Scope.} The family must be fixed independently of the certification sample,
or conditionally fixed using a disjoint selection split. Uncertified patterns are
handled by declining. Continuing to acquire at an uncertified pattern changes which
cases reach the later patterns, so it is a different complete policy and must be
represented and certified as one.

Family-wise control makes the event that $\widehat{\mathcal{W}}$ contains a true null
have probability at most $\delta$, so on its complement every listed pair satisfies
$R_\Sset(d)\le\alpha$ at once and an $\hat d$ chosen afterwards by any rule still lands
in a set whose members are all safe. Because the hypotheses are indexed by pairs,
certifying $(d,\Sset)$ certifies the risk of $d$ at a pattern $d$ itself produces.

Mondrian conditioning, exact binomial validation and family-wise testing are each
standard, and we claim no novelty for them individually. The contribution is
identifying the object they must be applied to once acquisition determines the
conditioning group. Learn-Then-Test already licenses selecting a configuration on the
data that grades it, but it certifies $R_\Sset(d)$ for a group fixed independently of
that data, and \cref{rem:sharp} shows this need not bound risk at the group the
policy actually produces. Indexing the hypotheses by pairs transports the guarantee to the group realized at deployment, which is why the certificate must be simultaneous over the family rather than pointwise at one configuration. A deployment that reaches an uncertified pattern and keeps buying rather than declining is a different complete policy, so its certificate no longer describes it; admitting that behavior costs $53$ further error-free cases per pattern at $\alpha{=}\delta{=}0.1$, $|\mathcal{S}|{=}8$ (\cref{app:continue,cor:continue}). The certified object thus constrains the deployment. On CMU-MOSEI no single cutoff is safe at every pattern, so the deployed policy withholds answers by declining rather than by buying more.

\begin{remark}[How small a group may be]\label{rem:cost}
A group can deploy a cutoff only if its validation half holds at least
$n_{\min}=\lceil\ln\delta/\ln(1-\alpha)\rceil$ answered cases, whatever the score,
model or data law; at $\alpha{=}\delta{=}0.1$, $n_{\min}{=}22$. Abstention on small
patterns is therefore forced by the finite-sample guarantee rather than by a loose
analysis, and the finest affordable partition is the finest one clearing the floor
in every cell (\cref{prop:coarsen,cor:adaptcoarsen}).
\end{remark}

\noindent
Writing $\kappa_\alpha=\ln(1/(1-\alpha))$, a correction dividing $\delta$ by a
family of $F$ hypotheses costs $\lceil\ln F/\kappa_\alpha\rceil$ cases per pattern on top
of $n_{\min}$ (\cref{prop:cost}). When the family is indexed by one scalar that
price can be avoided outright.

Because the hypotheses within one pattern are indexed by a single scalar, they can
instead be tested as a fixed sequence in an order fixed by data disjoint from the
certification sample: walk the grid in that order, spend the pattern's whole level
at each step, stop at the first non-rejection. That corrects only the
$|\mathcal{S}|$ patterns and leaves the grid free (\cref{prop:seq}), so the
pointwise, generic and sequential requirements are $22$, $71$ and $42$ error-free
answered cases. A poor ordering costs power, never validity.

\paragraph{The acquisition policy used here.}
Everything above holds for any routing satisfying \cref{asm:polcal}, so the acquisition rule is a design choice made inside that freedom, not part of the guarantee. We make one: buy small \emph{coalitions} rather than single sources, because residual-set reduction is not submodular on complementary inputs, so a policy that buys a source only while that source alone helps stalls exactly there (\cref{prop:synergy}). That motivates the choice; \cref{sec:exp} measures when it pays and when it does not.

Both guarantees target $y^{\mathrm{full}}$ by default. An audited bound on the
reference's own error rate (\cref{prop:truelabel}) and direct ground-truth
calibration (\cref{ssec:truelabel}) are the two routes to the true label. Validity never depends
on reference strength, attainable answer rate does.

\section{Experiments}\label{sec:exp}

\paragraph{Setup.}
These are controlled tests of policy-created calibration groups, conditional risk, complementarity and fragmentation, not a state-of-the-art comparison. Calibration must replay the modalities each policy would acquire, so the datasets are ones in which nominally missing sources remain recoverable; performance under naturally occurring missingness is not established. Four real datasets carry the paper: CMU-MOSEI \citep{zadeh2018mosei}, IEMOCAP \citep{busso2008iemocap} and MHEALTH \citep{banos2014mhealth} under masking, plus the PTB-XL clinical protocol. The wearable recordings are partitioned two ways (MHEALTH into three sensor packs for the risk-control tables, MHEALTH-5 into five for the coalition tables), and each claim uses only the datasets that support its acquisition setting (\cref{app:suite}).

\subsection{True-label correctness after adaptive acquisition}\label{ssec:truelabel}
Calibrating directly against ground truth, after each policy has run to its final
state, gives the strongest of the three guarantees (\cref{tab:truelabel}). Where the reference is weak the posterior inherits its errors and no cutoff passes
the true-label test, so the system abstains (CMU-MOSEI and IEMOCAP answer
$0.000$). Where it is strong, a usable operating point appears: RouteCert answers $46.7\%$ of
cases at $1.0\%$ true-label error on MHEALTH-5, and every arm sits well under the
$0.10$ cap. RouteCert answers the most while the forced arm attains the lowest
worst-pattern risk, and \cref{tab:maintl} gives the full sweep.

\begin{table}[t]
\centering\small
\setlength{\tabcolsep}{3.5pt}
\begin{tabular}{@{}llccc@{}}
\toprule
\textbf{Benchmark (ref.\ acc.)} & \textbf{Policy} & \textbf{Ans.} & \textbf{Risk} & \textbf{Worst} \\
\midrule
MHEALTH-5 ($0.97$) & cost-aware & $0.393$ & $0.013$ & $0.023$ \\
MHEALTH-5 ($0.97$) & forced & $0.435$ & $0.011$ & $0.020$ \\
MHEALTH-5 ($0.97$) & \textbf{RouteCert} & $\mathbf{0.467}$ & $\mathbf{0.010}$ & $0.022$ \\
\midrule
CMU-MOSEI, IEMOCAP & every rule & $0.000$ & abst. & abst. \\
\bottomrule
\end{tabular}
\caption{\textbf{Correctness, not decision stability.} RouteCert-Risk calibrated on
ground truth after end-to-end acquisition (per-pattern $\alpha{=}\delta{=}0.1$, five
seeds $\times$ $50$ resplits), the one risk column here that is a true-label
guarantee. A strong reference (in parentheses) admits a usable operating point and a
weak one abstains, ``abst.''\ marking an empty answered set.}
\label{tab:truelabel}
\end{table}

\paragraph{When coalitions are worth buying.}
Whether a coalition beats the best single source is predictable before deployment
from the \emph{needs-a-pair fraction}, the share of cases some pair resolves but no
single source does. Where it is substantial the coalition guarantees $+3$ to $+12$
points more, already net of the calibration cost of an extra terminal pattern, since
each needs a pool clearing $n_{\min}$ (\cref{tab:maintl}).

\subsection{Selective acquisition on a clinical ECG benchmark}\label{ssec:ptbxl}
PTB-XL
\citep{wagner2020ptbxl,goldberger2000physionet} tests the certificate itself on a
clinically meaningful acquisition order. A cardiologist-assigned myocardial-infarction (MI) label is predicted at three nested stages (lead~I; leads~I and~II; the eight independent signals of a 12-lead recording) at cumulative ordinal costs $0.1$, $1.0$, $5.0$. The leads are
recorded together and withheld retrospectively, so acquisition is simulated rather
than prospective, and routing is a fixed confidence cascade over stage-specific
models, not the coalition loop of \cref{sec:method}.
The protocol was fixed before folds~9 and~10 were opened, so this is a development result and external confirmation remains necessary. Patient-disjoint official folds separate fitting, threshold selection, certification
(fold~9, $n{=}1942$) and one held-out test (fold~10, $n{=}1904$).

Nine policy--pattern hypotheses, three terminals crossed with an adaptive family and two fixed-stage controls, are certified together by Holm, a step-down multiple-testing correction, at $\delta{=}0.1$ against an $\alpha{=}0.1$ cap. On the held-out fold the certified policy answers $71.2\%$ of patients ($1355/1904$,
$95\%$ CI $[0.691,0.732]$) at $7.4\%$ disagreement ($[0.060,0.089]$), at mean prespecified ordinal cost $2.442$ of $5.0$ ($48.8\%$), a stipulated cost scale rather than measured clinical burden. All three terminals certify and no registered fixed-stage policy dominates it on both cost and answered fraction (\cref{fig:ptbxlflow}, audited in \cref{fig:ptbxl}).

The aggregate hides a class asymmetry that matters clinically: the policy answers $76.7\%$ of non-MI patients and is correct on $98.5\%$ of those it answers, but answers only $53.5\%$ of MI patients and is correct on $65.8\%$ of those, so roughly a third of answered MI cases are wrong. The guarantee is terminal-pattern conditional, not disease-class conditional: this demonstrates workflow certification, not diagnostic readiness.

\begin{figure*}[t]
\centering
\includegraphics[width=0.82\textwidth]{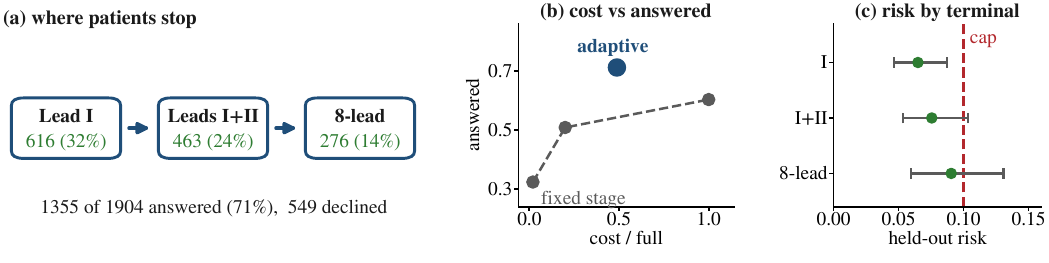}
\caption{\textbf{Selective acquisition on held-out PTB-XL patients.}
(a)~Every patient starts at lead~I; those the certificate can answer stop there and
the rest continue, with patients it still cannot guarantee declined. (b)~The adaptive
policy answers more than any fixed stage at or below its cost, so no registered
fixed-stage control dominates it. (c)~Held-out point estimates at each terminal lie
below the $\alpha{=}0.1$ cap, shown with descriptive intervals.}
\label{fig:ptbxlflow}
\end{figure*}

\subsection{Per-pattern versus global recalibration}\label{ssec:safety}\label{ssec:reliability}
The alternative to conditioning on the policy-created pattern is not to condition at
all, recalibrating globally on post-acquisition scores as BCEA does after a fixed
policy \citep{xu2026bcea}. Our \textbf{global recalibration} arm applies the same
RouteCert-Risk rule marginally over pooled patterns, so the arms differ only in what they
condition on.  A global threshold's realized risk is the answer-weighted average of the per-pattern risks, so it is a property of the traffic mixture, which adaptive acquisition itself produces. Conditioning removes that dependence.

Conditioning controls the quantity a pooled threshold cannot, and
\cref{tab:pervsglobal} reports the cost. At equal per-pattern $\delta$ it is close to
free; held to the same family-wise budget it is not, and on MHEALTH, whose per-pattern
pools are the thinnest, almost nothing remains certifiable.  Pooling mixes easy patterns with hard ones, so global recalibration exceeds the
$0.10$ cap on two of three benchmarks, while per-pattern calibration holds every
answered pattern at target, each at level $\delta$, and abstains on the hardest, so patterns that can be guaranteed answer freely instead of being held to a cutoff set by the hardest pattern in the pool (\cref{tab:answerrate}).

\begin{table}[t]
\centering\small
\setlength{\tabcolsep}{2.9pt}
\begin{tabular}{@{}lcccccc@{}}
\toprule
& \multicolumn{2}{c}{\textbf{pooled}} & \multicolumn{2}{c}{\textbf{per-pattern}} &
\multicolumn{2}{c}{\textbf{per-pat.}} \\
& \multicolumn{2}{c}{\textbf{(one test)}} & \multicolumn{2}{c}{\textbf{($\delta$ each)}} &
\multicolumn{2}{c}{\textbf{(simul.)}} \\
\cmidrule(lr){2-3}\cmidrule(lr){4-5}\cmidrule(lr){6-7}
\textbf{Benchmark} & \textbf{ans.} & \textbf{worst} & \textbf{ans.} & \textbf{worst} &
\textbf{ans.} & \textbf{worst} \\
\midrule
CMU-MOSEI & $0.342$ & $0.145$ & $0.350$ & $\mathbf{0.034}$ & $0.305$ & $\mathbf{0.024}$ \\
IEMOCAP   & $0.189$ & $0.064$ & $0.203$ & $\mathbf{0.004}$ & $0.188$ & $\mathbf{0.000}$ \\
MHEALTH   & $0.374$ & $0.123$ & $0.362$ & $\mathbf{0.009}$ & $0.014^{\dagger}$ & $\mathbf{0.002}^{\dagger}$ \\
\bottomrule
\end{tabular}
\caption{\textbf{Answered fraction and worst-pattern risk against a $0.10$ cap.} All
arms share scores, models and calibration rule. The middle columns are not
budget-matched against the pooled one (pooling spends one test at $\delta$,
per-pattern conditioning one per pattern), so they are pointwise, not simultaneous.
The last pair holds all patterns jointly at $90\%$, by Bonferroni with selection
corrected too on CMU-MOSEI and IEMOCAP (Bonf.$^\ast$, \cref{tab:simulfull}).
$^{\dagger}$MHEALTH's simultaneous entries are a separate Bonferroni pass (validation
only, $\delta/7$), not the matched run behind its other columns. Risk is disagreement
with the reference model ($0.50$ CMU-MOSEI, $0.60$ IEMOCAP held-out), so on those two
it bounds decision stability, not correctness.}
\label{tab:pervsglobal}
\end{table}

\paragraph{Controlling every pattern at once.}
Holding all patterns jointly is affordable where pools have room (CMU-MOSEI
$0.350\to0.305$, IEMOCAP $0.203\to0.188$), and \cref{prop:cost} says why the
thinnest pool pays more (MHEALTH $0.362\to0.014$). Two levers spend the budget: correcting the \emph{test},
where step procedures never fall below Bonferroni, and correcting the
\emph{selection} of the cutoff (\cref{tab:simulfull}).

\subsection{Decision stability, and where the coalition gain comes from}\label{ssec:e2e}
A guarantee proved for a fixed pattern is only useful if it survives the system choosing that pattern, so each of six policies runs cold-start at budget two and RouteCert-Risk is applied at the pattern it reached (\cref{tab:coalition}). Every policy's realized and worst-pattern risk stayed at or below target, and the coalition guarantees more cases than the strongest single-source policy among our matched score-level implementations at the same budget. The gain is not an artifact of how a pair is executed, since an open-loop commit control matches it and a replanning two-step planner guarantees less.

\begin{table}[t]
\centering\small
\setlength{\tabcolsep}{3.6pt}
\begin{tabular}{@{}lcccc@{}}
\toprule
& \textbf{best} & \textbf{RouteCert} & & \textbf{worst} \\
\textbf{Benchmark} & \textbf{single} & \textbf{coalition} & \textbf{gain} & \textbf{disagr.} \\
\midrule
CMU-MOSEI & $0.485$ & $\mathbf{0.588}$ & $+0.103$ & $0.054$ \\
IEMOCAP & $0.201$ & $\mathbf{0.324}$ & $+0.123$ & $0.023$ \\
\bottomrule
\end{tabular}
\caption{\textbf{Coalition look-ahead under actual adaptive acquisition.} Guaranteed
answer rate at a matched two-acquisition budget, the gain paired over five training
runs per benchmark rather than over resplits ($95\%$ CI $[0.025,0.181]$ and
$[0.026,0.220]$). Worst-pattern risk holds under the $0.10$ cap throughout. These are
decision-stability numbers, the reference being ${\approx}0.50$ and ${\approx}0.60$
accurate and true-label error among answered cases $0.458$ and $0.378$.}
\label{tab:coalition}
\end{table}

\subsection{Letting acquisition read the calibrated cutoff}\label{ssec:aware}
\Cref{thm:aware} permits the stopping rule to consult the calibrated cutoff when
the certificate is simultaneous over the policy family. The arms share models,
splits and purchase order and differ only in what decides when to stop, threshold-free
routing never reading calibration while the rest route on a cutoff certified by an
empirical-risk rule, a pointwise exact test, Holm, or the sequential rule. All answer only at certified patterns, without altering where the policy stops.

Simultaneous certification removes the cap exceedances pointwise selection produces (\cref{tab:breach}); certified sequentially, a per-pattern cutoff family recovers $77$--$95\%$ of the gap to threshold-free routing and lifts MHEALTH from $0.033$ to $0.363$. Every simultaneously certified arm then holds the cap on all $50$ resplits (\cref{tab:aware}).

\paragraph{Score quality with the route held
fixed.}\label{ssec:robust}\label{ssec:dominance}\label{ssec:coalitions}\label{ssec:shift}
The certificate is modular: applied to EDDI \citep{ma2019eddi}, a DyMo-style prototype reward \citep{du2026dymo}, expected utility and information gain, it caps every route at $0.034$ or below where a global threshold exceeds the cap on all (\cref{tab:confbase}).

\section{Conclusion and Limitations}
\label{sec:scope}\label{sec:conclusion}
RouteCert certifies adaptive modality acquisition when the policy determines its
terminal group, either by keeping routing calibration-blind or by jointly certifying
fixed policy--pattern pairs. This yields pattern-wise guarantees but trades conditional resolution against
calibration size. Our evidence is limited to replayable masking and a retrospective,
single-site, class-asymmetric ECG simulation (\cref{app:ptbxl}), so naturally missing
data and externally validated clinical deployment remain future work.

\paragraph{Generative-AI disclosure.} A generative-AI assistant was used to polish
and organize the manuscript and the code. The author takes full responsibility for all
content.

\clearpage
\appendix

\setcounter{section}{0}\setcounter{table}{0}\setcounter{figure}{0}
\setcounter{algorithm}{0}\setcounter{equation}{0}
\setcounter{theorem}{0}\setcounter{proposition}{0}\setcounter{corollary}{0}
\setcounter{assumption}{0}\setcounter{definition}{0}\setcounter{remark}{0}
\renewcommand{\thesection}{S\arabic{section}}
\renewcommand{\thetable}{S\arabic{table}}
\renewcommand{\thefigure}{S\arabic{figure}}
\renewcommand{\thealgorithm}{S\arabic{algorithm}}
\renewcommand{\theequation}{S\arabic{equation}}
\renewcommand{\theproposition}{S\arabic{proposition}}
\renewcommand{\thetheorem}{S\arabic{theorem}}
\renewcommand{\thecorollary}{S\arabic{corollary}}
\renewcommand{\theassumption}{S\arabic{assumption}}
\renewcommand{\thedefinition}{S\arabic{definition}}
\renewcommand{\theremark}{S\arabic{remark}}

\section*{Appendix}
\noindent This appendix contains the benchmark suite and provenance, the assumptions
and complete proofs, the deployable acquisition-scoring algorithm, every secondary and
control experiment, and the reproducibility configuration. Its own results, sections,
tables, figures and equations carry an ``S'' prefix, so no number is shared with the
main text.

\paragraph{Computing infrastructure.} All experiments ran on a Slurm cluster; each job used one node with two Intel Xeon Gold 6226 processors (Cascade Lake, $24$ cores at $2.70$\,GHz), $4$ allocated CPU cores, $32$\,GB of host RAM and one NVIDIA V100 (16\,GB), under Red Hat Enterprise Linux~9. The PTB-XL study ran in the site PyTorch~25.10 container (Python~3.12, PyTorch~2.9); all other experiments used Python~3.10 with PyTorch~2.1.0. Verification of the shipped artifacts needs no GPU and under $2$\,GB of RAM.

\paragraph{Number of runs behind each reported result.} Every number in the paper comes from one of four run counts, stated here once so each result is unambiguous. Main-paper Table~1 and the per-pattern risk-control tables use one trained model per benchmark evaluated over $50$ calibration/evaluation resplits. The end-to-end acquisition tables and the coalition comparisons use five independently retrained seeds, each itself averaged over $50$ resplits. The race, pair-race, $K$-sensitivity and coarsening studies use three retrained seeds. The PTB-XL clinical study is a single prespecified run on one held-out fold (fold~10, $n{=}1904$) with three checkpoint seeds fixed before the fold was opened; its confidence intervals are exact binomial, not across-run. Where a table reports a standard deviation it is across retrained seeds, never across resplits.

\section{Full benchmark suite}\label{app:suite}
\Cref{tab:suite} summarizes the suite. This work uses four real datasets. Three are multimodal benchmarks evaluated under controlled masking, CMU-MOSEI \citep{zadeh2018mosei}, IEMOCAP \citep{busso2008iemocap} and MHEALTH \citep{banos2014mhealth}, the last evaluated in two configurations (MHEALTH, three sensor packs, for the risk-control results; MHEALTH-5, five per-sensor streams, for the coalition results). The fourth is PTB-XL, the staged clinical electrocardiogram protocol of \cref{app:ptbxl}, which is a headline benchmark rather than a supporting one. Together they span language, audio, visual, physiological-sensor and clinical-electrocardiogram sources. Earlier drafts also reported five two-source datasets and a synthetic redundancy control; we have removed them, since with a single source missing every acquisition rule coincides and they therefore test nothing about acquisition.

\begin{table}[htbp]
\centering\small
\begin{tabular}{@{}llcc@{}}
\toprule
\textbf{Benchmark} & \textbf{Acquirable sources} & $M$ & $C$ \\
\midrule
CMU-MOSEI & text, audio, visual & 3 & 7 \\
IEMOCAP & text, audio, visual & 3 & 4 \\
MHEALTH & three sensor packs & 3 & 12 \\
\bottomrule
\end{tabular}
\caption{The three controlled-masking benchmarks, their acquirable sources, source count $M$, and class count $C$. All $M$ sources are acquirable from a cold start. PTB-XL, the fourth dataset, uses the separate staged electrocardiogram protocol of \cref{app:ptbxl}.}
\label{tab:suite}
\end{table}

\section{Completion-model robustness and the true-label screen}\label{app:methoddetails}
Under \cref{asm:polcal}, the conformal coverage of \Cset\ (\cref{thm:coverage}) is distribution-free whatever the completion model's quality, since calibration uses the same model. What depends on how well \gen\ covers plausible completions is only the semantic reading of \Lset\ as ``answers genuinely in play'', which remains a heuristic. \paragraph{Crudeness of the residual-set screen.} The size $|\Lset|$ counts
distinct argmax classes across the $K$ completions and discards their probability
mass, so a rare draw can inflate it and its expected value grows with $K$ (mean
$|\Lset|$ rises $1.96{\to}3.46$ on CMU-MOSEI from $K{=}5$ to $50$). This is why the
guarantee never rests on $|\Lset|$: per-pattern coverage and selective-risk validity
are distribution-free in \gen\ (\cref{thm:coverage,prop:rcps}) and reported
separately from the $|\Lset|$ screen, which only routes acquisition. A
probability-weighted residual measure would be less sensitive to rare draws and is
left to future work. Finite $K$ can also miss a decision-changing completion, which
is quantifiable rather than merely a caveat.

\begin{proposition}[True-label guarantee via the reference gap]\label{prop:truelabel}
Let $\varepsilon_{\mathrm{ref}}(\Sset)=\Pr[\,y^{\mathrm{full}}\neq y\mid\textsc{answer},\Sset\,]$ be the reference model's true-label error rate on the answered region of pattern $\Sset$. Whenever the reference-relative selective risk is controlled at level $r_{\mathrm{o}}$ with probability at least $1-\delta$ (the validated $\alpha$ of \cref{prop:rcps}; for the population bound of \cref{thm:selrisk} take $\delta{=}0$), the true-label selective risk satisfies
\[
  \Pr[\,\yhat\neq y\mid\textsc{answer},\Sset\,]\;\le\;r_{\mathrm{o}}+\varepsilon_{\mathrm{ref}}(\Sset).
\]
A point estimate of $\varepsilon_{\mathrm{ref}}(\Sset)$ alone does not guarantee this bound. Given an independent labeled audit of $m$ i.i.d.\ answered cases drawn under the deployed policy in $\Sset$, let $\bar\varepsilon_{\mathrm{ref}}(\Sset)$ be the one-sided Clopper--Pearson upper confidence bound at level $1-\delta'$ on the audit error rate. Then, with probability at least $1-\delta-\delta'$ over the calibration and audit draws,
\[
  \Pr[\,\yhat\neq y\mid\textsc{answer},\Sset\,]\;\le\;r_{\mathrm{o}}+\bar\varepsilon_{\mathrm{ref}}(\Sset),
\]
a genuine high-probability true-label guarantee; the audit point estimate concentrates at rate $O(1/\sqrt m)$, so the slack $\bar\varepsilon_{\mathrm{ref}}(\Sset)-\varepsilon_{\mathrm{ref}}(\Sset)=O\!\bigl(\sqrt{\log(1/\delta')/m}\bigr)$ vanishes with the audit size.
\end{proposition}

\begin{proof}
On the answer event in $\Sset$, $\{\yhat\neq y\}\subseteq\{\yhat\neq y^{\mathrm{full}}\}\cup\{y^{\mathrm{full}}\neq y\}$, since if $\yhat=y^{\mathrm{full}}$ and $y^{\mathrm{full}}=y$ then $\yhat=y$. Taking probabilities conditional on $\{\textsc{answer},\Sset\}$ and subadditivity gives $\Pr[\yhat\neq y\mid\textsc{answer},\Sset]\le\Pr[\yhat\neq y^{\mathrm{full}}\mid\textsc{answer},\Sset]+\varepsilon_{\mathrm{ref}}(\Sset)\le r_{\mathrm{o}}+\varepsilon_{\mathrm{ref}}(\Sset)$, on the calibration event where \cref{prop:rcps} holds (probability $\ge1-\delta$; probability $1$ for the population bound of \cref{thm:selrisk}). The audit errors are i.i.d.\ Bernoulli$(\varepsilon_{\mathrm{ref}}(\Sset))$, so the one-sided Clopper--Pearson bound satisfies $\varepsilon_{\mathrm{ref}}(\Sset)\le\bar\varepsilon_{\mathrm{ref}}(\Sset)$ with probability $\ge1-\delta'$, and its exact binomial tail gives the stated $O(\sqrt{\log(1/\delta')/m})$ slack. A union bound over the two failure events (each at most $\delta$ and $\delta'$) yields $r_{\mathrm{o}}+\bar\varepsilon_{\mathrm{ref}}(\Sset)$ as an upper bound with probability at least $1-\delta-\delta'$.
\end{proof}

\noindent
This makes the reference-to-correctness gap a measured quantity: on MHEALTH-5
($\rcls$ accuracy $0.97$) the answered-region
$\varepsilon_{\mathrm{ref}}\approx0.05$, so the $\alpha$-guaranteed reference risk
already implies a true-label cap near $\alpha+0.05$; on CMU-MOSEI ($0.51$)
$\varepsilon_{\mathrm{ref}}\approx0.44$, so the guarantee reports decision stability,
not correctness, and the direct true-label layer (\cref{app:ablation}) is the route
to a correctness guarantee there.

\Cref{prop:truelabel} also makes the choice of where to answer a
guaranteed one, converting the reference-relative default into a tunable true-label
deployment.

\noindent
Here $\gamma$ names the correctness the practitioner requires, and RouteCert answers on
exactly the patterns it can guarantee at that level. Where the reference is strong
(MHEALTH-5, $\bar\varepsilon_{\mathrm{ref}}\approx0.05$) every pattern clears
$\gamma\approx\alpha+0.05$, so the audited true-label deployment is essentially the
full system; where it is weak (CMU-MOSEI,
$\bar\varepsilon_{\mathrm{ref}}\approx0.44$) only a loose $\gamma$ is attainable, so
the rule declines rather than emit an answer it cannot guarantee.

\noindent
The bound is immediate from independence of the $K$ draws. It gives the deployed
$K$ an operational reading: at $K{=}25$ and $C{=}7$, every completion mode of mass
$\ge0.2$ is recalled with probability $\ge0.97$ (union bound), whereas at $K{=}15$ it
is $\approx0.75$ and at $K{=}50$ it is $\ge0.9999$. This matches the $K$-sweep, where
the acquisition score settles by $K\approx25$, and quantifies one limitation of the residual set, missing a rare but decision-changing
completion, as a decaying function
of $K$ and the mode mass $\rho_{\mathrm{mode}}$; it does not affect the guarantee, which is
distribution-free in $\gen$ at any finite $K$.
A stress test shows both sides: a noisier \gen\ widens \Lset\ and makes RouteCert more cautious, while a mode-collapsed, overconfident \gen\ over-emits with higher risk before calibration. Per-pattern calibration absorbs a stationary completion bias, since the same \gen\ produces the scores at calibration and test, so any systematic distortion folds into the calibrated threshold and coverage holds regardless of \gen's quality; it does not absorb a calibration-to-deployment shift in \gen.

The default guarantee is scored against the full-information reference model, the right target when the operative question is whether the evidence in hand suffices. Some deployments instead need a guarantee against the true label; RouteCert then adds one more conformal screen, calibrated the same per-pattern way but scored on held-out true labels, and commits only when both layers agree (reported separately, \cref{sec:guarantee}). Turning the second on cuts true-label error among answered cases on CMU-MOSEI (reference only $0.51$-accurate) from $0.44$ to $0.07$ (\cref{app:ablation}), at a lower answer rate.

\Cref{tab:truelabel} lists the direct true-label operating points after end-to-end acquisition: the two strong-reference benchmarks admit a usable point, and every weak reference abstains rather than answer without a guarantee.

\section{Assumptions and proofs}\label{app:theory}
\paragraph{The calibration procedure.}
Both Phase-2 layers share the routing and the split but differ in the calibrated
quantity, the conformal threshold $\tau_\Sset$ (RouteCert-Cov, set-valued) against the
validated confidence cutoff $\lambda_\Sset$ (RouteCert-Risk, a scalar rejector), never
interchanged. \textbf{(0)}~Fit the completion model, reference model and policy
$\pi$ on a training split disjoint from calibration. \textbf{(1)}~Route every
calibration point through the Phase-1 loop of \cref{alg:routecertmain} (no calibrated
quantity is read), recording its terminal pattern $\Sset_{\mathrm{fin}}$, and group by
that pattern. \textbf{(2a) RouteCert-Cov:} within each pattern set $\tau_\Sset$ to the
$\lceil(1-\alpha)(n_\Sset+1)\rceil$-th smallest nonconformity score
(\cref{asm:polcal}); at test, answer iff the answer set is a singleton.
\textbf{(2b) RouteCert-Risk:} within each pattern, split its calibration points in half,
choose a confidence cutoff on the selection half by any rule, and deploy it as
$\lambda_\Sset$ iff its exact Clopper--Pearson bound on the validation half is
$\le\alpha$, else abstain (\cref{prop:rcps}); at test, answer iff the top
completion-averaged confidence clears $\lambda_\Sset$. Every risk-capped result uses (2b) and the
coverage tables (2a), both on the threshold-free routing the proofs require.

\begin{table*}[htbp]
\centering\small
\begin{tabular}{@{}lll@{}}
\toprule
\textbf{Object} & \textbf{Role} & \textbf{Carries a guarantee} \\
\midrule
$\Lset$ (from \gen, \rcls) & routes acquisition (Phase 1) & no; heuristic screen \\
$\phat$, $\yhat$ & prediction and confidence & scored by Phase 2 \\
$\tau_\Sset$ (RouteCert-Cov) & per-pattern conformal set & yes (\cref{thm:coverage}) \\
$\lambda_\Sset$ (RouteCert-Risk) & validated confidence cutoff & yes (\cref{prop:rcps}); default \\
$y^{\mathrm{full}}$ (reference) & default target & decision stability, not correctness \\
$y$ (true label) & correctness target & direct calibration or audit (\cref{prop:truelabel}) \\
\bottomrule
\end{tabular}
\caption{Reader's map: each object's role and whether it carries a guarantee. The
residual set routes but is never guaranteed; the guarantees attach to the Phase-2
layers at the terminal pattern.}
\label{tab:objects}
\end{table*}

\paragraph{Augmented exchangeability (making ``by construction'' precise).}
Write each point as an augmented pair $(x,\omega)$, where $\omega$ collects the
auxiliary randomness of its $K$ completion draws. The Phase-1 map reads only
$(x,\omega)$, the costs, and the budget, never a calibration label or a
threshold, so $\Sset_{\mathrm{fin}}=\pi(x,\omega)$ is a fixed measurable function
of the augmented pair. The completion model, reference model, and $\pi$ are fit on data
disjoint from calibration, so they act as constants with respect to the
calibration and test draw. Assume the augmented pairs
$(x_1,\omega_1),\dots,(x_{n},\omega_{n}),(x_{\ast},\omega_{\ast})$ are i.i.d. The
points with $\pi(\cdot)=\Sset$ are then i.i.d.\ draws from the conditional law
given $\{\pi=\Sset\}$, hence exchangeable: selecting on the value of a fixed
function of each pair preserves exchangeability within the selected group.
Because each point's group label depends only on its own augmented pair,
conditioning the i.i.d.\ sequence on the per-point event $\{\pi=\Sset\}$ leaves the
selected points i.i.d.; the within-point coupling between a point's score and its
routing therefore does not affect the within-group rank uniformity. The
nonconformity scores inside $\Sset$ are therefore exchangeable and the rank of a
test score routed to $\Sset$ among the $n_\Sset$ calibration scores there is
uniform. This is all the coverage step uses; nothing is assumed beyond i.i.d.\
sampling and a threshold-free, disjoint-fit routing map. This is the precise
content of ``by construction'': the randomized routing is absorbed into the
augmented pair, so no completion-kernel or seed-matching condition beyond i.i.d.\
sampling is needed. In particular, finite $K$ degrades only the deployed policy's
decisions, not the guarantee: the same finite-$K$ loop routes calibration and test
alike, so \cref{thm:coverage,thm:selrisk} hold for the deployed finite-$K$ system,
not merely in a $K\to\infty$ limit.

\paragraph{Coverage (\cref{thm:coverage}).}
Fix a terminal pattern $\Sset$ and condition on the threshold-free event
$\{\Sset_{\mathrm{fin}}=\Sset\}$. By the augmented-exchangeability argument above,
the routing map $\pi(x,\omega)$ reads only the augmented pair, the costs, and the
budget, never a calibration label or the threshold $\tau_\Sset$, and the
completion model, reference model, and $\pi$ are fit on data disjoint from calibration.
Hence the $n_\Sset$ calibration points routed to $\Sset$ and the test point, when
it is also routed to $\Sset$, are i.i.d.\ draws from the same conditional law
given $\{\pi=\Sset\}$, and their nonconformity scores are exchangeable. The
per-pattern calibration count is itself random, so we make the conditioning
explicit: conditional on the test point routing to $\Sset$ and on the
count $N_\Sset=n_\Sset$, the $n_\Sset$ calibration scores in $\Sset$ and the test
score are exchangeable draws from the conditional score law given
$\{\pi=\Sset\}$ (routing is a fixed per-point function, so conditioning on the
per-point routing events and their count preserves exchangeability); the rank
argument below holds for every value of $n_\Sset$
(including $n_\Sset=0$, where $\tau_\Sset=+\infty$ and coverage is trivial), and
the unconditional statement follows by integrating over $N_\Sset$ (tower
property). Let
$s_1,\dots,s_{n_\Sset}$ be the calibration scores and $s_\ast$ the test score
routed to $\Sset$. By exchangeability the rank of $s_\ast$ among the $n_\Sset+1$
scores is uniform on $\{1,\dots,n_\Sset+1\}$ (with continuous scores ties occur
with probability zero; when scores can tie, the conservative quantile with the
$s_\ast\le\tau_\Sset$ acceptance keeps the rank super-uniform, so the $1-\alpha$
bound holds with no tie-breaking randomization). The threshold $\tau_\Sset$ is the
$\lceil(1-\alpha)(n_\Sset+1)\rceil$-th order statistic of $s_1,\dots,s_{n_\Sset}$ (or
$+\infty$ when this index exceeds $n_\Sset$, so $\Cset$ contains all labels and
coverage is trivial), so
$s_\ast\le\tau_\Sset$ whenever the rank of $s_\ast$ is at most
$\lceil(1-\alpha)(n_\Sset+1)\rceil$, an event of probability
$\lceil(1-\alpha)(n_\Sset+1)\rceil/(n_\Sset+1)\ge1-\alpha$ (by the rank
uniformity above). Since the calibration
target $t$ ($y^{\mathrm{full}}$ for the default layer, $y$ for the optional
screen) lands in $\Cset$ exactly when $s_\ast\le\tau_\Sset$ (by the definition
of the conformal set $\Cset$), we conclude
$\Pr[\,t\in\Cset\mid\Sset\,]\ge1-\alpha$. This is standard Mondrian split
conformal with the terminal pattern as the group.

The main text's coverage theorem implies a population selective-risk bound of the
divided form, which we state and prove here; it is not deployed, and every
headline risk number comes from \cref{prop:rcps}.

\begin{corollary}[Selective risk under adaptive acquisition]\label{thm:selrisk}
Under \cref{asm:polcal}, for every final missing-source pattern
$\Sset=\Sset_{\mathrm{fin}}$ with $\Pr[\textsc{answer}\mid\Sset]>0$, when the
RouteCert-Cov layer commits to a single answer $\yhat$ (the \textsc{answer} event is
$\{|\Cset|{=}1\}$; an empty or larger answer set declines; patterns with
$\Pr[\textsc{answer}\mid\Sset]=0$ abstain identically),
\[
  \Pr\!\bigl[\,\yhat\neq y^{\mathrm{full}}\,\big|\,\textsc{answer},\,\Sset\,\bigr]
  \;\le\;\min\!\Bigl\{1,\ \frac{\alpha}{\Pr[\textsc{answer}\mid\Sset]}\Bigr\} ,
\]
the denominator being that pattern's own answer rate. Calibrating at the tighter level $\alpha'=\alpha\amin$ keeps answered-case risk at most $\alpha$ whenever every pattern answers with probability at least $\amin$.
\end{corollary}

\paragraph{Proof of \cref{thm:selrisk}.}
Write $E=\{\yhat\neq y^{\mathrm{full}}\}$ for the error event and
$A=\{\textsc{answer}\}=\{|\Cset|{=}1\}$ for the answer event. Fix $\Sset$.
The event $A$ does depend on $\tau_\Sset$; no independence is claimed for it.
Instead we bound the joint and divide (definition of conditional probability):
\[
  \Pr[E\mid A,\Sset]
  =\frac{\Pr[E,A\mid\Sset]}{\Pr[A\mid\Sset]}
  \;\le\;\frac{\alpha}{\Pr[A\mid\Sset]} .
\]
The numerator is $\le\alpha$ for the following reason. The system answers only
when $|\Cset|{=}1$, and it answers with the unique element $\yhat$ of $\Cset$.
Therefore the joint event $E\cap A$
implies $y^{\mathrm{full}}\neq\yhat$ while $\Cset=\{\yhat\}$, that is
$y^{\mathrm{full}}\notin\Cset$, a miscoverage event. By the coverage step applied
with target $t=y^{\mathrm{full}}$, $\Pr[\,y^{\mathrm{full}}\notin\Cset\mid\Sset\,]\le\alpha$,
so $\Pr[E,A\mid\Sset]\le\alpha$ (monotonicity of probability under the event
inclusion).
Dividing by $\Pr[A\mid\Sset]$ gives $\alpha/\Pr[A\mid\Sset]$; since the left side is a probability it is also at most $1$, which yields the stated $\min\{1,\cdot\}$ form. The
conditioning event $\Sset$ is legitimate precisely because routing is
threshold-free; the threshold-dependent answer choice only ever appears
inside the joint probability, never as a conditioning event, so the
coverage step still applies under the conditioning on $\Sset$. For the population
form, the per-pattern joint bound $\Pr[E,A\mid\Sset]\le\alpha$ aggregates (law
of total probability over the patterns) as
\[
  \Pr[E,A]=\sum_\Sset\Pr[\Sset]\,\Pr[E,A\mid\Sset]\;\le\;\alpha ,
\]
and dividing by $\Pr[A]$ gives $\Pr[E\mid A]\le\alpha/\Pr[A]$.

\paragraph{A counterexample: reusing the calibration cutoff for routing can break coverage.}
A minimal instance makes the failure exact. Let calibration and test nonconformity
scores be i.i.d.\ $U[0,1]$ with calibration quantile $\tau$, route a test point into
$\Sset$ iff its score exceeds $\tau$, and guarantee it by $s\le\tau$ within $\Sset$.
Conditional on $\Sset$ the score is above $\tau$, so realized coverage is
$\Pr[\,s\le\tau\mid\Sset\,]=0$. The acquisition loop realizes a milder version of the
same coupling. Suppose it kept acquiring until the current pattern's calibrated set
became a singleton, reading $\tau_\Sset$. Whether a test case stays in $\Sset$ then
depends on $\tau_\Sset$, computed from $\Sset$'s own calibration points, so on a draw
where those scores run low the threshold is permissive, extra cases stop in $\Sset$,
and the rank of a test score among calibration scores is no longer uniform. No routing
decision in \cref{alg:routecertmain} reads $\tau$, so the coupling cannot arise there.
We measure it over ten calibration/test resplits (one trained reference and completion
model throughout, $\alpha{=}0.10$), recomputing $\tau_\Sset$ from each split's
calibration half. The threshold-free arm routes by $|\Lset|$ and the confidence-routed
arm stops when that split's own $\tau_\Sset$ makes the set a singleton, so only whether
test-time routing consults the threshold differs. The confidence-routed rule realizes
marginal coverage $0.888\pm0.007$, below the $0.90$ target, against $0.901\pm0.007$
threshold-free. The effect is modest but consistent and in the predicted direction, and
its scope is exactly this reuse, since a confidence policy frozen on an independent
split avoids the coupling and stays valid (\cref{sec:guarantee}).

To show the gap is not intrinsically small, we repeat the test on controlled
synthetic data where the rate $\rhocpl$ of confidently-wrong-under-missingness
cases (the $K$-averaged posterior peaked on a non-target class while $|\Lset|>1$,
the exact configuration the coupling exploits) is a free knob
(\cref{tab:synthce}, same per-pattern thresholds, $\alpha{=}0.10$, $40{\times}3$
resplits). Threshold-free routing holds the $0.90$ target at every $\rhocpl$, overall
and on every fixed pattern's aggregated coverage ($0.899$--$0.901$). Confidence-routed stopping degrades smoothly
as the coupling strengthens: overall coverage falls from $0.84$ at $\rhocpl{=}0.1$ to
$0.45$ at $\rhocpl{=}0.5$, and its worst-pattern coverage falls to $0.01$--$0.15$, far below target. At $\rhocpl{=}0$ it does not under-cover, which confirms the breach
is caused by the coupling the proof identifies and not by the stop rule alone.
CMU-MOSEI sits near the small-$\rhocpl$ end, which is why its real-data gap is modest;
the mechanism itself is not bounded.

\subsection{Proofs of the four structural results}\label{app:newproofs}

\paragraph{Sharpness of the routing condition (\cref{rem:sharp}).}
Let calibration scores $s_1,\dots,s_n$ and a test score $s_\ast$ be i.i.d.\
$U[0,1]$, and let $\tau$ be the $\lceil(1-\alpha)(n+1)\rceil$-th smallest
calibration score. Route the test point into group $\Sset$ exactly when
$s_\ast>\tau$, and inside $\Sset$ emit the answer set
$\Cset=\{y:s(y)\le\tau\}$. Conditional on $\{\Sset_{\mathrm{fin}}=\Sset\}$ we
have $s_\ast>\tau$ by definition of the routing rule, so the target is excluded
with probability one and
$\Pr[\,t\in\Cset\mid \Sset_{\mathrm{fin}}=\Sset\,]=0$. Every other hypothesis of
\cref{thm:coverage} holds: the points are i.i.d., the models are fit on disjoint
data, the score is a fixed measurable function, and the routing is a deterministic measurable map of the input \emph{and} the calibration set.
Since coverage is $0$ rather than merely below $1-\alpha$, the example rules out
any nontrivial distribution-free guarantee under \emph{unrestricted}
calibration-dependent routing: no bound survives if the dependence is left
unconstrained. It does \emph{not} rule out approximate validity under an
explicit quantitative restriction on that dependence, such as a bound on mutual information, a coupling or total-variation budget, algorithmic stability, or
differential privacy could each buy a degraded but nonvacuous guarantee, and
\citet{barber2023beyond} is exactly such a route, bounding the loss by a
computable coupling term that our construction drives to its maximum. The
correct reading is therefore that arbitrary calibration dependence admits no
distribution-free guarantee, so approximate validity requires the dependence to
be quantified; \cref{thm:aware} takes the different route of removing the
dependence from the certified object instead of bounding it. Two further remarks
delimit the claim. First, it concerns worst cases over data laws, as
distribution-free statements always do; a particular calibration-dependent
policy may be harmless on a particular law. Second, it does not say routing must ignore \emph{all} data; \cref{thm:aware} identifies which data it must ignore. The measured decay between the two regimes
is \cref{tab:synthce}.

\begin{table}[htbp]
\centering\small
\begin{tabular}{@{}lcccc@{}}
\toprule
& \multicolumn{2}{c}{threshold-free (RouteCert)} & \multicolumn{2}{c}{confidence-routed} \\
\cmidrule(lr){2-3}\cmidrule(lr){4-5}
$\rhocpl$ & overall & worst & overall & worst \\
\midrule
$0.0$ & $0.90$ & $0.90$ & $1.00$ & $1.00$ \\
$0.1$ & $0.90$ & $0.90$ & $0.84$ & $0.10$ \\
$0.2$ & $0.90$ & $0.90$ & $0.74$ & $0.15$ \\
$0.3$ & $0.90$ & $0.90$ & $0.64$ & $0.01$ \\
$0.5$ & $0.90$ & $0.90$ & $0.45$ & $0.02$ \\
\bottomrule
\end{tabular}
\caption{Synthetic counterexample (target $0.90$, $120$ resplits): threshold-free
versus confidence-routed stopping as the rate $\rhocpl$ of
confidently-wrong-under-missingness cases grows; same thresholds, only the stop
rule differs. ``worst'': minimum over patterns of fixed-pattern coverage across
resplits, the quantity \cref{thm:coverage} controls. Threshold-free holds the
target at every $\rhocpl$ ($0.899$--$0.901$); confidence-routed
coverage declines with $\rhocpl$.}
\label{tab:synthce}
\end{table}

\paragraph{Sample splitting, and why it is not enough (\cref{sec:guarantee}).}
Let $D_A$ and $D_B$ be disjoint and let $\pi=\pi(\cdot\,;D_A)$ be any measurable
function of $D_A$. Condition on $D_A$: then $\pi$ is a fixed measurable map from
the augmented pair $(x,\omega)$ to a terminal pattern, exactly the object
\cref{asm:polcal} requires, and $D_B$ is independent of $D_A$ and identically
distributed with the test point. So \cref{asm:polcal} holds verbatim with $D_B$
as the calibration set, \cref{thm:coverage,prop:rcps} give their conclusions
conditionally on $D_A$, and since the bounds do not depend on the value of
$D_A$, averaging preserves them. A $K$-fold version routes each test point by
the policy fitted on $\bigcup_{j\neq k}D_j$ for an independently drawn $k$ and
calibrates on $D_k$. This is standard, and it is what the frozen-confidence arm
does. Its limits are equally standard: the policy is graded on data it never
saw, so it can be learned but it still cannot consult the deployed guarantee,
and every fold spent on routing is a fold not spent on calibration.
\Cref{thm:aware} gives up neither.

\paragraph{Calibration-aware acquisition (\cref{thm:aware}).}
Fix the grid $\Lambda$ and the pattern set $\mathcal{S}$ \emph{before} looking at
the calibration data, and let $d_\lambda$ be the acquisition-and-answer policy
described in the main text. For each $\lambda$, $d_\lambda$ is a deterministic
measurable map from $(x,\omega)$ to a purchase sequence, a terminal pattern, and an
answer-or-decline decision; nothing about it depends on the calibration set.
Consequently $R_\Sset(\lambda)$, the selective risk among the cases $d_\lambda$
answers at terminal pattern $\Sset$, is a fixed population quantity, one per pair
$(\lambda,\Sset)\in\Lambda\times\mathcal{S}$, and the null
$H_{\lambda,\Sset}:R_\Sset(\lambda)>\alpha$ is a fixed hypothesis.

Let $n_{\lambda,\Sset}$ and $k_{\lambda,\Sset}$ be the number of calibration
cases $d_\lambda$ answers at $\Sset$ and how many it gets wrong. Because the
calibration points are i.i.d.\ and $d_\lambda$ does not read them, these errors
are binomial at rate $R_\Sset(\lambda)$ conditional on $n_{\lambda,\Sset}$, so
$p_{\lambda,\Sset}=\Pr[\mathrm{Bin}(n_{\lambda,\Sset},\alpha)\le k_{\lambda,\Sset}]$
is super-uniform under $H_{\lambda,\Sset}$. Apply any procedure controlling the
family-wise error rate at $\delta$ over these hypotheses and let
$\widehat{\mathcal{W}}$ be its rejected set. We use Holm \citep{holm1979},
which is valid under arbitrary dependence. That matters: the independence
established above is \emph{across patterns}, whereas within one
pattern the $p$-values for different $\lambda$ are computed from overlapping
answered sets and are strongly dependent, so a step-up procedure such as
Hochberg is not licensed over $\Lambda\times\mathcal{S}$ and we do not use one
there. By definition of family-wise control,
$\Pr[\widehat{\mathcal{W}}\text{ contains a true null}]\le\delta$, which is
precisely
$\Pr[\,R_\Sset(\lambda)\le\alpha\ \forall(\lambda,\Sset)\in\widehat{\mathcal{W}}\,]\ge1-\delta$.

The deployment step is where the assumption is actually relaxed. Choose
$\hat\lambda$ by \emph{any} rule, including one that reads
$\widehat{\mathcal{W}}$, the estimated risks, the answer rates, or the whole calibration set, and deploy $d_{\hat\lambda}$ with its answers restricted to
the patterns certified for $\hat\lambda$. On the $1-\delta$ event every deployed
pair satisfies $R_\Sset(\lambda)\le\alpha$ simultaneously, so whichever pair the
selection lands on is covered. No independence between the selection and the
calibration set is used, because the guarantee being transferred is uniform over
the family rather than pointwise at one pair. That is exactly what fails in
\cref{rem:sharp}: there the guarantee holds at the realized threshold marginally
over calibration draws, and selecting on that same draw does not preserve it.

\paragraph{The per-pattern policy family, and where it runs out of data.}
The family $\{d_\lambda\}$ of the main text carries one cutoff at every pattern, a
poorer class than the per-pattern cutoffs a calibration-blind baseline fits, and that
class, not the certificate, is what costs answer rate. A richer family is available at
no cost in validity. Split the calibration half into a selection part and a disjoint
certification part. On the selection part
compute, for each mask $\Sset$ and each level $\eta$ in a pre-registered grid
$\mathcal{R}$, the most permissive cutoff whose \emph{empirical} risk at that
mask is at most $\eta$; call the resulting vector $\lambda(\eta)$. Each $d_{\eta}$ buys while confidence at the current pattern is below
$\lambda_\Sset(\eta)$ and answers when it clears, so the family
$\{d_\eta:\eta\in\mathcal{R}\}$ is a deterministic function of the selection part
alone. Conditional on that part it is pre-registered, so \cref{thm:aware} applies
with the certification part as its calibration set, and averaging over the selection
draw preserves the bound. We use $\mathcal{R}=\{0.01,\dots,0.20\}$ in steps of
$0.01$ and Holm over $\mathcal{R}\times\mathcal{S}$.

Empirically this recovers most of what the single-cutoff family gives up:
answered fraction rises from $0.371$ to $0.879$ on CMU-MOSEI, $0.487$ to
$0.780$ on IEMOCAP, closing $85\%$ and $74\%$
and $89\%$ of the distance to calibration-blind routing, with no breach on any
of the $50$ resplits of any benchmark. It does not overtake the blind baseline,
which remains $6$--$11$ points ahead.

On MHEALTH the arm abstains outright, and the arithmetic of \cref{rem:cost}
says why. Splitting the calibration half again leaves a certification quarter;
Holm's leading threshold over $|\mathcal{R}|\times|\mathcal{S}|=20\times8$
hypotheses is $\delta/160=6.25\times10^{-4}$, and clearing it with no observed
errors needs $n\ge\ln(6.25\times10^{-4})/\ln(0.9)\approx70$ answered cases in
whichever pattern is rejected first. MHEALTH's $174$ certification cases spread
across seven initial patterns leave roughly $25$ apiece, so no pair is ever
rejected and no $\eta$ is certified. MHEALTH survives on its larger pools
only because its answers concentrate in fewer patterns. The binding constraint
on calibration-aware routing is thus the size of the certification pool relative
to the family, which is the same floor that governs how fine a calibration
partition may be.

\paragraph{The restriction must suppress answers, not redirect traffic.}
The certified quantity
$R_\Sset(\lambda)$ is the risk among the cases \emph{that $d_\lambda$ itself}
routes to $\Sset$ and answers there. A deployment that, on reaching an
uncertified pattern, keeps buying instead of declining is a different policy,
say $d_{\lambda}'$, and for a later pattern $\Sset_2$ one has in general
\[
  R_{\Sset_2}^{d_\lambda'}(\lambda)\ \neq\ R_{\Sset_2}^{d_\lambda}(\lambda),
\]
so the certificate computed for $(\lambda,\Sset_2)$ under $d_\lambda$ says
nothing about $d_\lambda'$. A two-type example makes the gap concrete. Type~A
cases are confidently wrong at $\Sset_1$, so $d_\lambda$ stops and answers there
and $(\lambda,\Sset_1)$ is correctly refused certification; type~B cases pass
through $\Sset_1$ to $\Sset_2$, where they are always right, so
$(\lambda,\Sset_2)$ certifies on a population of pure~B. Redirect the type~A
cases onward and $\Sset_2$ receives A${}+{}$B at deployment, with risk bounded
only by the type~A error rate. Nothing about the calibration computation
detects this, because the calibration never saw the modified policy.

Declining is safe for the complementary reason: it changes no case's purchase
sequence, so every pattern receives exactly the population it was certified on,
and restricting to a subset of answered pairs can only remove answered cases
from patterns that are no longer used. The alternative, if one genuinely wants
the continue-behavior, is to make it part of the policy: index candidates by
$(\lambda,A)$ where $A$ is a pre-registered answer/continue/decline map, and
control the family-wise error over $\Lambda\times\mathcal{A}\times\mathcal{S}$.
That is valid for any $\mathcal{A}$ fixed in advance, at a multiplicity cost of
$|\mathcal{A}|$; we do not pursue it, and the arm reported in \cref{ssec:aware}
declines.

\paragraph{The deployment floor (\cref{rem:cost}).}
Fix a group and suppose its validation half yields $n$ answered cases with $k$
errors. The deployed rule accepts iff the one-sided Clopper--Pearson upper
confidence bound at level $1-\delta$, the $(1-\delta)$-quantile of
$\mathrm{Beta}(k+1,n-k)$, is at most $\alpha$. That quantile is nondecreasing in
$k$, so it is minimized at $k=0$, where $\mathrm{Beta}(1,n)$ has quantile
function $1-(1-p)^{1/n}$ and the bound equals $1-\delta^{1/n}$. Acceptance
therefore requires $1-\delta^{1/n}\le\alpha$, i.e.\ $\delta^{1/n}\ge1-\alpha$,
i.e.\ $\tfrac1n\ln\delta\ge\ln(1-\alpha)$. Both logarithms are negative, so this
rearranges to $n\ge\ln\delta/\ln(1-\alpha)$. At $\alpha=\delta=0.1$ the
right-hand side is $\ln(0.1)/\ln(0.9)=21.85$, giving $n_{\min}=22$. The bound is
attained: $n=22$ error-free answered cases accept. Nothing about the score,
model, or data law enters, so a group whose answered validation count falls below
$n_{\min}$ cannot deploy however good its predictor is, and a partition whose
cells cannot be expected to clear the floor cannot be used. Selecting the finest
partition that clears it on a held-out selection split is exactly the
data-adaptive backoff of \cref{cor:adaptcoarsen}, which is valid because the
selection part is disjoint from the guarantee part.

\paragraph{Simultaneous control dominating Bonferroni.}
Condition on the routing assignments (which fix every pattern count $n_\Sset$)
and on every pattern's selection half (which fixes every selected cutoff
$\lambda_\Sset$, hence which patterns have true selective risk above $\alpha$,
i.e.\ which nulls are true). Two facts follow. First, each pattern's validation
half is a fresh sample never used to choose its own cutoff, so under its null
the exact one-sided binomial $p$-value
$p_\Sset=\Pr[\mathrm{Bin}(n_\Sset^{\mathrm{ans}},\alpha)\le k_\Sset]$ is
super-uniform, which is what any family-wise procedure requires. Second, the
patterns \emph{partition} the calibration set, so the validation halves are
disjoint subsets of an i.i.d.\ sample and the $\{p_\Sset\}$ are mutually
independent under the conditioning. Holm's step-down procedure controls the
family-wise error rate at $\delta$ for super-uniform $p$-values under arbitrary
dependence, and Hochberg's step-up procedure \citep{hochberg1988} does so under independence; both
therefore apply conditionally, and since the bound $\delta$ does not depend on
the conditioning variables, averaging preserves it.

Domination is immediate \emph{at a matched selection rule}. Bonferroni deploys
exactly the patterns with $p_\Sset\le\delta/|\mathcal{S}|$. If there are $k$
such patterns, the $k$ smallest $p$-values are all at most
$\delta/|\mathcal{S}|$, and Holm's thresholds $\delta/(|\mathcal{S}|-j)$ at
steps $j=0,\dots,k-1$ are each at least $\delta/|\mathcal{S}|$, so Holm passes
through all $k$ steps and deploys at least the same patterns; Hochberg, being a
step-up procedure with the same threshold sequence, deploys a superset of
Holm's. Neither can therefore answer less than Bonferroni, and both can answer
more.

The qualifier matters in practice, because ``Bonferroni'' names two different
procedures here. One corrects only the \emph{test}, comparing $p_\Sset$ to
$\delta/|\mathcal{S}|$ while choosing each candidate cutoff at $\delta$; the
other, which the earlier version of this paper reported, also tightens
\emph{selection}, choosing the candidate cutoff at $\delta/|\mathcal{S}|$ as
well. The second selects more conservative cutoffs that clear their test more
easily, so it is not dominated by the step procedures and can answer more than
they do. \Cref{tab:simulfull} reports both, and the domination above is the
like-for-like comparison against the first.

For contrast we also implement a fixed-sequence arm: order the patterns by a
routing-measurable statistic (decreasing pool size $n_\Sset$), test each at the
\emph{full} level $\delta$, and stop at the first failure. That is also valid, since a family-wise error requires the first bad pattern in the order to pass its own single test, which happens with probability at most $\delta$, and it splits no level at all. It is nonetheless the weakest arm in
\cref{tab:simulfull}, because with patterns of comparable size the ordering carries
almost no information about which will pass, so the sequence truncates after
$0.1$--$0.9$ patterns on average and forfeits the rest. Spending the error
budget adaptively, as Holm and Hochberg do, beats spending it in a fixed order.

\paragraph{The divided bound and the deployed rule.}
The bound $\alpha/\Pr[\textsc{answer}\mid\Sset]$ is weak on patterns that answer
rarely (on CMU-MOSEI the text-missing pattern abstains entirely under the
deployed rule, \cref{tab:answerrate}). We therefore do not rely on the divided form in the
risk-capped experiments. There, each per-pattern threshold is set by the
split-select/validate rule of \cref{prop:rcps}: the pattern's calibration cases
are halved uniformly at random. A candidate confidence cutoff is chosen on the selection half (by the
Clopper--Pearson scan, though any selector measurable in that half is allowed).
The single chosen cutoff is then tested once on the held-out validation half
with an exact one-sided Clopper--Pearson bound at confidence $1-\delta$, deployed
iff the bound is $\le\alpha$ and abstaining otherwise.

\noindent\textbf{Proof of \cref{prop:rcps}.} Fix a pattern $\Sset$ and, for a cutoff
$\lambda$, define over the validation half
\[
  A_i(\lambda)=\mathbf 1\{c_i\ge\lambda\},\qquad
  E_i(\lambda)=\mathbf 1\{\yhat_i\neq t_i\}\,A_i(\lambda),
\]
with $N_\lambda=\sum_i A_i(\lambda)$ answered cases, $K_\lambda=\sum_i E_i(\lambda)$
errors, and population selective risk
$r_\Sset(\lambda)=\Pr[\,\yhat\neq t\mid c\ge\lambda,\Sset\,]$,
where $c_i$ is the completion-averaged top confidence and $t$ the calibration
target. The candidate cutoff $\lambda_\Sset$ is a fixed function of the selection
half, so conditional on the selection half and on $N_{\lambda_\Sset}=n$, the
validation points are i.i.d.\ from the conditional law given
$\{c\ge\lambda_\Sset,\Sset\}$ and
\[
  K_{\lambda_\Sset}\;\big|\;N_{\lambda_\Sset}=n
  \;\sim\;\mathrm{Binomial}\bigl(n,\ r_\Sset(\lambda_\Sset)\bigr)
\]
(a sum of conditionally i.i.d.\ Bernoulli error indicators), with $n=0$ forcing
abstention. Deploying iff the one-sided Clopper--Pearson
bound satisfies $U_{\mathrm{CP}}(K_{\lambda_\Sset},N_{\lambda_\Sset};\delta)\le\alpha$
therefore gives, for every value of the selection half and every $n\ge0$,
$\Pr[\,r_\Sset(\lambda_\Sset)>\alpha\ \text{and deploy}\mid\cdot\,]\le\delta$
(Clopper--Pearson exactness),
and integrating over both preserves the bound (tower property): one exact binomial test, no
monotonicity assumption on $r_\Sset(\cdot)$, no multiple-testing structure. \hfill$\square$

Patterns whose validation half cannot guarantee
(a Clopper--Pearson bound at $\delta{=}0.1$ needs at least $22$ error-free
answered cases to reach $\alpha{=}0.10$) abstain entirely; abstention on hard
patterns is thus the correct output of risk control, not a loophole in it. The
divided bound is the clean population statement, and the deployed calibration is
the finite-sample tool. This rule controls the selective risk directly rather
than through the coverage bound of \cref{thm:coverage}: it deploys the scalar
cutoff $\lambda_\Sset$ in place of the coverage quantile $\tau_\Sset$ of
\cref{asm:polcal}, on the same threshold-free calibration split, so the
exchangeability that licenses it is unchanged.

\paragraph{Pair look-ahead versus myopic strict-shrink acquisition (\cref{prop:synergy}).}
The statement and proof are for the population residual set
$\Lset_\infty(x_\Sset)=\{c:\Pr_{z\sim\gen(\cdot\mid x_\Sset)}[\argmax_y\rcls(x_\Sset,z)_y=k]>0\}$,
the support of the completion-induced decision distribution; the deployed
$\Lset$ is its $K$-sample estimate, and the transfer to the deployed loop is
quantified below. Consider the exclusive-or construction with two acquirable
sources $a$ and $b$,
each a single bit, and label $y=a\oplus b$. Conditioned only on the observed
sources at a state, $\Lset_\infty$ is the set of labels still
consistent with what has been read. With neither $a$ nor $b$ observed, both labels
remain possible, so $|\Lset_\infty|=2$. Observing $a$ alone leaves $y$ equal to either
$a\oplus0$ or $a\oplus1$ as $b$ ranges over its two values, so both labels remain
and $|\Lset_\infty|=2$; by symmetry observing $b$ alone also leaves $|\Lset_\infty|=2$.
Observing the pair $\{a,b\}$ fixes $y=a\oplus b$, so $|\Lset_\infty|=1$. Call an input on
which this holds a synergy-only input, and suppose such inputs occur with
probability $p>0$. A greedy one-at-a-time policy buys a source only while some
single source strictly shrinks $\Lset$. On a synergy-only input no single source
shrinks $\Lset$, so the greedy policy makes no purchase and stops with
$|\Lset|=2$, hence never resolves the input (never collapses $\Lset$ to a
singleton) regardless of how it scores
individual sources. Take the pair look-ahead policy to make the same purchases as the greedy policy on
every input some single source can settle, and on the synergy-only inputs to
evaluate the residual shrinkage of $\{a,b\}$, find it collapses $|\Lset|$ to $1$,
and buy the pair at any budget $\ge2$. By construction the two policies then agree
off the synergy event, so the resolved fractions differ only on the
synergy-only inputs, where the greedy policy resolves none and the look-ahead
policy resolves all. At any budget $\ge2$ the greedy policy therefore resolves a
fraction smaller by at least $p>0$ (the mass of the synergy event). (Resolution, $|\Lset|{=}1$, is the Phase-1
screen; the Phase-2 guarantee then applies to the resolved cases as in the
coverage and risk-control results.) A policy forced to exhaust its budget can reach
the pair by exhaustion; the race in \cref{ssec:coalitions} includes that forced
variant as a control, and it still falls short.

Finite-$K$ transfer. The deployed loop computes the $K$-sample estimate
$\Lset$, not $\Lset_\infty$, so with probability $2(1/2)^K=2^{1-K}$ per
evaluation (independence of the $K$ completion draws) all $K$ draws of a
balanced binary state land on one class and the
sampled set is spuriously a singleton. That bounds a single
evaluation, and a route makes $N_{\mathrm{eval}}$ of them (one per visited state,
candidate group, and outer draw: at budget $B$ with $r_{\max}{=}2$,
$N_{\mathrm{eval}}\le B\,(1+D(M+\binom{M}{2}))$), and a union bound over the
$N_{\mathrm{eval}}$ evaluations gives
recovery of every population mode of mass $\ge\rho_{\mathrm{mode}}$ along the whole route with
probability at least $1-N_{\mathrm{eval}}\,C(1-\rho_{\mathrm{mode}})^K$. At $K{=}15$, $C{=}2$,
$\rho_{\mathrm{mode}}{=}1/2$, and $N_{\mathrm{eval}}\le200$ this failure probability is below
$1.3\%$ per input, and the observed $100\%$ XOR resolution
is consistent with it; the population statement of
\cref{prop:synergy} therefore transfers to the deployed loop up to this margin,
and the calibration guarantees are unaffected either way, since calibration and
test run the identical finite-$K$ procedure.

\begin{proposition}[Residual-set reduction is non-submodular, so complementarity is invisible at order 1]\label{prop:synergy}
There exist decision problems on which $f(\{a\})=f(\{b\})=0$ yet $f(\{a,b\})=1$
for every input (an exclusive-or: either source alone is uninformative, the pair
decides), so $f$ is non-submodular, and every greedy one-at-a-time policy that
buys only while some single source strictly shrinks $\Lset_\infty$ stops
immediately on those inputs, resolving the residual set on a strictly smaller
fraction at the same budget than a policy that looks ahead over pairs. In the
balanced construction the population completion-averaged posterior is uniform
over the two labels, so the stalled policy answers with error $1/2$ on those
inputs. The residual-set statement is unconditional; the gap carries through to
the \emph{guaranteed} decision under the further conditions that
$\alpha<1/2$, so that no cutoff can certify the stalled policy, that the
reference model is exact on the completed input, and that each pattern's
calibration pool clears the floor of \cref{rem:cost}, so that the pair policy's
zero-risk answer is certifiable at all. Absent the first, the unresolved policy
may itself certify; absent the third, both policies abstain and the gap closes
trivially.
\end{proposition}

\noindent
The proof is immediate from the construction. Scope: this rules out the myopic
strict-shrink policy only. A policy that exhausts its budget can reach the pair, and does: at $M{=}5$ the forced-spend control reaches $0.786$ resolution
against the coalition's $0.850$, a real but much smaller margin than against the
myopic policy's $0.726$ (\cref{app:mfive}). The obstruction is the look-ahead
order rather than the score: expected utility and information gain are flat on
the same construction, and the argument recurs at every order.

\subsection{The deployed rule}\label{app:rcps}

\Cref{prop:rcps} states the deployed rule; its proof is
in \cref{app:theory}.

\noindent
A pattern therefore answers only with a cutoff that survived one exact test on
data it was not chosen on, and data-poor patterns abstain rather than guess. The
proof is in \cref{app:theory}.

\begin{figure}[t]
\centering
\includegraphics[width=\columnwidth]{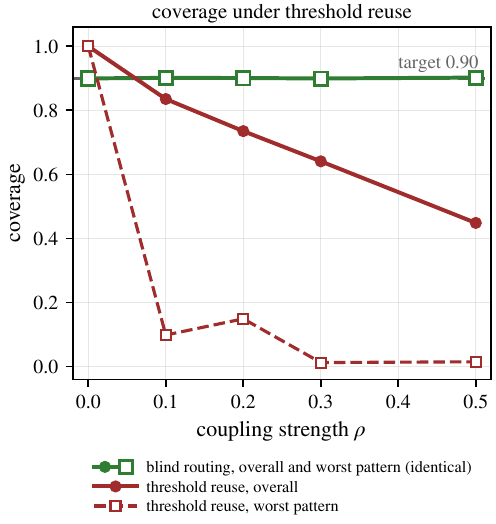}
\caption{\textbf{The boundary is approached continuously, not at a corner.}
Coverage against the strength $\rho_{\mathrm{cpl}}$ of the coupling between stopping and the calibrated threshold
($40$ resplits $\times$ $3$ seeds). Blind routing holds the $0.90$ target overall
and on its worst pattern, which take the same value at every coupling strength and
are therefore drawn as one curve. Threshold reuse decays to $0.45$ overall and its
worst pattern to $0.01$ (\cref{rem:sharp,tab:synthce}).}
\label{fig:coupling}
\end{figure}

\subsection{What validation costs}\label{app:cost}

\begin{proposition}[What a certificate costs]\label{prop:cost}
Write $\kappa_\alpha=\ln\!\bigl(1/(1-\alpha)\bigr)$. A pattern answering $n$ calibration cases
with no error is certifiable by the exact binomial test at level $\delta'$ if and
only if $n\ge\ln(1/\delta')/\kappa_\alpha$. Hence the smallest certifiable pattern holds
$n_{\min}=\lceil\ln(1/\delta)/\kappa_\alpha\rceil$ cases under a pointwise certificate,
$n_{\min}+\lceil\ln F/\kappa_\alpha\rceil$ under any correction dividing $\delta$ by a family
size $F$, and $n_{\min}+\lceil\ln|\mathcal{S}|/\kappa_\alpha\rceil$ under \cref{prop:seq}
applied to $\mathcal{R}\times\mathcal{S}$. A grid of policies therefore costs
$\ln|\mathcal{R}|/\kappa_\alpha$ error-free cases per pattern generically, up to rounding, and
none sequentially. At $\alpha{=}\delta{=}0.1$, $\kappa_\alpha{=}0.105$ and $n_{\min}{=}22$; with
$|\mathcal{R}|{=}20$ and $|\mathcal{S}|{=}8$ the three requirements are $22$, $71$
and $42$ cases. Since $\kappa_\alpha=\alpha+O(\alpha^2)$, the price of a large family grows
like $\ln F/\alpha$ as the risk target tightens (\cref{app:newproofs}).
\end{proposition}

\noindent
\subsection{A sequential test for a scalar-indexed policy family}\label{app:seq}

\Cref{thm:aware} needs family-wise control over $\mathcal{D}\times\mathcal{S}$ and
nothing more, so any valid procedure may supply it. A generic one pays for every
member of $\mathcal{D}$. When that family happens to be indexed by a single scalar,
as both of ours are, the index removes the grid from the correction entirely.

\begin{proof}[Proof of \cref{prop:cost}]
With $k=0$ errors the exact one-sided $p$-value is
$\Pr[\mathrm{Bin}(n,\alpha)\le 0]=(1-\alpha)^n$, so the test rejects at level
$\delta'$ exactly when $(1-\alpha)^n\le\delta'$, i.e.\ $n\ln(1-\alpha)\le\ln\delta'$,
i.e.\ $n\ge\ln(1/\delta')/\ln\bigl(1/(1-\alpha)\bigr)=\ln(1/\delta')/\kappa_\alpha$. Any $k>0$
only raises the $p$-value, so this is the smallest certifiable $n$. Substituting
$\delta'=\delta$, $\delta'=\delta/F$ and $\delta'=\delta/|\mathcal{S}|$ and using
$\ln(F/\delta)=\ln(1/\delta)+\ln F$ gives the three requirements; subtracting the
second from the third gives the grid's price, $\ln|\mathcal{R}|/\kappa_\alpha$ up to rounding
when $F=|\mathcal{R}||\mathcal{S}|$. Finally $\kappa_\alpha=-\ln(1-\alpha)=\alpha+\alpha^2/2+\cdots$,
so $\ln(1/\delta')/\kappa_\alpha\sim\ln(1/\delta')/\alpha$ as $\alpha\to0$.
\end{proof}

\noindent
The proposition is what makes the choice of correction a budgeting question rather
than a matter of taste: at $\alpha=\delta=0.1$ a pattern needs $22$ error-free
answered cases to be certifiable at all, $42$ under \cref{prop:seq} over a
$20\times8$ family, and $71$ under a generic correction over the same family. At
$\alpha=0.05$ the same three become $45$, $86$ and $144$, which is why tightening
the risk target penalises large families disproportionately.

\begin{proposition}[Sequential certificate]\label{prop:seq}
Fix, from data independent of the certification sample, a scalar-indexed family
$\{d_\eta\}_{\eta\in\mathcal{R}}$ of complete acquisition policies, a pattern set
$\mathcal{S}$, and an ordering of $\mathcal{R}$ within each pattern. Testing
$H_{\eta,\Sset}:R_\Sset(\eta)>\alpha$ along that order with any valid $p$-value,
each at level $\delta/|\mathcal{S}|$, and stopping at the first non-rejection
controls the family-wise error rate over $\mathcal{R}\times\mathcal{S}$ at $\delta$,
so \cref{thm:aware} applies to the rejected pairs.
\end{proposition}

\begin{proof}[Proof of \cref{prop:seq}]
Fix $\Sset$ and let $j^\star$ be the first index in that pattern's ordering whose
null $H_{\eta_{\Sset,j^\star},\Sset}$ is true; if there is none the sequence makes no
false rejection. Because the sequence stops at its first non-rejection, any false
rejection in pattern $\Sset$ requires $H_{\eta_{\Sset,j^\star},\Sset}$ itself to be
rejected, an event of probability at most $\delta/|\mathcal{S}|$ by validity of
the $p$-value at that single hypothesis. The ordering is a function of data
independent of the certification sample, so this bound is unconditional. A union
bound over the $|\mathcal{S}|$ patterns gives $\delta$.
\end{proof}

\noindent
Two things are worth noting. The correction now scales with $|\mathcal{S}|$ alone,
so by \cref{rem:cost} the leading threshold needs
$\ln(\delta/|\mathcal{S}|)/\ln(1-\alpha)$ error-free answered cases in the first
pattern rejected rather than $\ln(\delta/|\Lambda||\mathcal{S}|)/\ln(1-\alpha)$;
at $\alpha=\delta=0.1$ with $|\mathcal{S}|=8$ and $|\Lambda|=20$ that is $42$
against $70$. And the ordering is where the family's structure enters: it must be
pre-specified, but it may be estimated, and in \cref{ssec:aware} we take it from
the selection half of the calibration data, which is disjoint from the
certification half by construction. A poor ordering costs power, never validity.
\Cref{tab:aware} reports what it buys.

\subsection{Continuing instead of declining}\label{app:continue}

\Cref{thm:aware} requires the deployed policy to suppress \emph{answers} at an
uncertified pattern rather than to keep buying, because continuing would send a
different population into the later patterns and their certificates are computed
under the deployed policy. That is a statement about one family, not about what is
certifiable: the continuing policy is simply a different member of a larger one.

\begin{corollary}[Continuing is available, and priced]\label{cor:continue}
For $A\subseteq\mathcal{S}$ let $d_{\lambda,A}$ be the policy that routes as
$d_\lambda$ except that at every pattern in $A$ it continues buying rather than
stopping. Each $d_{\lambda,A}$ is itself a complete acquisition policy, so if the
enlarged family $\{d_{\lambda,A}:\lambda\in\Lambda,\,A\subseteq\mathcal{S}\}$ is
fixed independently of the certification sample, \cref{thm:aware} applies to it
verbatim, and a certified $d_{\hat\lambda,\hat A}$ may continue at the patterns of
$\hat A$ with every answered pattern still carrying the $\alpha$ cap. By
\cref{prop:cost} the enlargement costs $\lceil|\mathcal{S}|\ln 2/\kappa_\alpha\rceil$
additional error-free answered cases per pattern under a generic correction: at
$\alpha=\delta=0.1$ with $|\mathcal{S}|=8$, that is $53$ cases on top of the $71$
the cutoff grid already requires.
\end{corollary}

\begin{proof}
Each $d_{\lambda,A}$ maps an input to purchases, a stopping point and an answer, so
the nulls $H_{d,\Sset}:R_\Sset(d)>\alpha$ are well defined over the enlarged family,
and \cref{thm:aware} asks only that the family be fixed independently of the
certification sample. The family size is multiplied by $2^{|\mathcal{S}|}$, and
\cref{prop:cost} turns that factor into
$\lceil\ln 2^{|\mathcal{S}|}/\kappa_\alpha\rceil=\lceil|\mathcal{S}|\ln 2/\kappa_\alpha\rceil$ extra
error-free cases in the first pattern rejected. Patterns in $A$ contribute no
answered cases under $d_{\lambda,A}$ and simply never appear as answering patterns.
\end{proof}

\noindent
So the choice between declining and continuing is a budgeting question of the same
kind as the choice of correction: enlarging the family is always sound and
\cref{prop:cost} says what it costs in calibration cases. Whether the enlarged
family is affordable is a property of the pool, not of the theorem.

\begin{table}[t]
\centering\small
\setlength{\tabcolsep}{3pt}
\begin{tabular}{@{}lccc@{}}
\toprule
\textbf{Certificate} & \textbf{MOSEI} & \textbf{IEMOCAP} & \textbf{MHEALTH} \\
\midrule
empirical risk       & $1.00$ & $1.00$ & $1.00$ \\
pointwise            & $0.12$ & $0.12$ & $0.02$ \\
simult., one cutoff  & $\mathbf{0.00}$ & $\mathbf{0.00}$ & $\mathbf{0.00}$ \\
simult., per-pattern & $\mathbf{0.00}$ & $\mathbf{0.00}$ & $\mathbf{0.00}$ \\
\bottomrule
\end{tabular}
\caption{Empirical counterpart of the validity boundary: the fraction of the $50$
resplits on which some answered pattern exceeds the $0.10$ selective-risk cap, when
acquisition stops on a cutoff taken from the same calibration half. The four arms share models, splits, purchase order and exact
$p$-values, and differ only in how the cutoff is certified. ``simult.'': simultaneous over the
policy family. This is an empirical stress-test frequency, not an estimate of the
failure probability the theorem bounds. Answer rates and worst-pattern risks are in \cref{tab:aware}.}
\label{tab:breach}
\end{table}

\begin{table*}[t]
\centering\small
\setlength{\tabcolsep}{3pt}
\begin{tabular}{@{}lcccccc@{}}
\toprule
& \multicolumn{2}{c}{\textbf{CMU-MOSEI}} & \multicolumn{2}{c}{\textbf{IEMOCAP}}
& \multicolumn{2}{c}{\textbf{MHEALTH}} \\
\cmidrule(lr){2-3}\cmidrule(lr){4-5}\cmidrule(lr){6-7}
\textbf{Stopping rule} & \textbf{ans.} & \textbf{brc} & \textbf{ans.} & \textbf{brc}
& \textbf{ans.} & \textbf{brc} \\
\midrule
\multicolumn{7}{@{}l}{\emph{routing never reads calibration; per-pattern control only}} \\
threshold-free      & $0.969$ & $0.08$ & $0.897$ & $0.04$ & $0.911$ & $0.00$ \\
\midrule
\multicolumn{7}{@{}l}{\emph{routing reads the calibration half}} \\
pointwise ($p\le\delta$)  & $0.410$ & $0.12$ & $0.534$ & $0.12$ & $0.861$ & $0.02$ \\
uniform, one cutoff       & $0.363$ & $\mathbf{0.00}$ & $0.365$ & $\mathbf{0.00}$ & $\mathbf{0.619}$ & $\mathbf{0.00}$ \\
uniform, per-pattern (Holm) & $0.914$ & $\mathbf{0.00}$ & $0.770$ & $\mathbf{0.00}$ & $0.033$ & $\mathbf{0.00}$ \\
\textbf{uniform, per-pattern (seq.)} & $\mathbf{0.938}$ & $\mathbf{0.00}$ & $\mathbf{0.777}$ & $\mathbf{0.00}$ & $0.363$ & $\mathbf{0.00}$ \\
\bottomrule
\end{tabular}
\caption{\textbf{Acquisition may read the calibrated cutoff, and what validation
costs.} ``ans.'': answered fraction; ``brc'': fraction of the $50$ resplits on
which some answered pattern exceeds the $0.10$ cap; bold marks the best certified
answer rate. The blind arms are calibrated per pattern at level $\delta$ each,
whereas every uniform arm controls all patterns at once. ``Per-pattern'' indexes a
cutoff vector by a scalar risk level, built on a selection half and certified on
the disjoint half; ``seq.'' certifies it by \cref{prop:seq} rather than Holm.
Spending the whole budget answers $0.995$--$0.999$; a cutoff frozen on a reserved third answers $0.720$--$0.839$ and breaches on up to $24\%$ of resplits; the empirical-risk arm answers $0.665$--$0.996$ and breaches on every resplit. Those arms, worst-pattern risk, purchases per case and an adversarial pointwise selection are in \cref{app:aware}. Answer rates here are not comparable to \cref{tab:maintl},
since every case may buy back all missing sources.}
\label{tab:aware}
\end{table*}

\subsection{The calibration-aware arms in full}\label{app:aware}

\Cref{tab:aware} reports seven arms from one matched pass (\texttt{calaware\_v6},
$50$ resplits, $\alpha=\delta=0.1$, grid $\Lambda$ of $15$ cutoffs, risk grid
$\mathcal{R}$ of $20$ levels).

\section{Error-vs-coverage results and the safety attribution}\label{app:aurc}

\subsection{Per-pattern answer rates}\label{app:answerrate}
We quantify how much answer rate survives the per-pattern risk control (here
``coverage'' is the selective-prediction answered fraction, not conformal coverage),
since the bound $\alpha/\Pr[\textsc{answer}\mid\Sset]$ is loose when a group answers
rarely. \Cref{tab:answerrate} reports both quantities per pattern on CMU-MOSEI. Where
it can guarantee, RouteCert answers ($0.99$ with nothing missing, about $0.35$ with audio
or vision gone); on text-missing cases its validation test fails and it abstains
entirely, the guarantee-or-abstain rule working as designed. It keeps per-group error
under the cap by declining exactly the cases with no guaranteed answer, not by
inflating risk on the cases it does answer. MHEALTH behaves the same way, with all
two-sensor dropouts abstaining.

\begin{table*}[htbp]
\centering\small
\begin{tabular}{@{}lcc@{}}
\toprule
\textbf{Missing pattern (CMU-MOSEI)} & \textbf{answer rate} & \textbf{error $\mid$ answer} \\
\midrule
nothing missing      & $0.99$ & $0.000$ \\
audio missing        & $0.36$ & $0.033$ \\
vision missing       & $0.34$ & $0.038$ \\
audio+vision missing & $0.05$ & $0.012$ \\
text missing (hardest) & $0.00$ & abstains \\
\bottomrule
\end{tabular}
\caption{Per-pattern answer rate and per-pattern error among answered cases on
CMU-MOSEI at the $10\%$ cap. RouteCert answers most cases where it can guarantee and
declines the hardest pattern (text missing), holding every answered group at or below
the cap.}
\label{tab:answerrate}
\end{table*}

\paragraph{Where the safety comes from.}
The per-pattern safety is conferred by the per-pattern (Mondrian) conformal
guarantee, which we credit as standard. That guarantee is modular, and we show it
directly (\cref{tab:confbase}). Give expected utility, information gain, and RouteCert each
the same per-pattern guarantee, and all three are safe (worst-pattern risk
$\le0.034$); give each a single global threshold, and all three breach the cap
($0.108$ to $0.128$). What caps each pattern's risk is the guarantee, not which source
to buy, so the safety result is not evidence that RouteCert's acquisition rule beats
expected utility or information gain; on raw acquisition quality the rules are at
parity.

\begin{table*}[htbp]
\centering\small
\begin{tabular}{@{}lcc@{}}
\toprule
\textbf{Acquisition rule (CMU-MOSEI)} & \textbf{global threshold} & \textbf{+ per-pattern cert.} \\
\midrule
expected utility   & $0.108$ & $0.028$ \\
information gain    & $0.125$ & $0.026$ \\
RouteCert               & $0.128$ & $0.034$ \\
\bottomrule
\end{tabular}
\caption{Worst-pattern reference-relative selective risk (target $0.10$) after each
acquisition rule buys one source, under a single global threshold versus the same
per-pattern (Mondrian) guarantee, mean over $50$ splits. Risks differ from
\cref{ssec:safety} because every rule here first buys one source, softening the
hardest pattern; the contrast is unchanged. Every rule is safe with the per-pattern
guarantee and breaches with the global threshold: the safety is the guarantee, not
the acquisition rule.}
\label{tab:confbase}
\end{table*} What is RouteCert-specific is twofold. First, the guarantee keeps holding under RouteCert's own
adaptive acquisition (\cref{thm:selrisk} and the reliability sweep,
\cref{ssec:reliability}): a per-pattern guarantee naively attached to a
confidence-selecting policy can break under that policy's adaptivity, which the
threshold-free buy signal prevents. Second, the coalition look-ahead reaches the
needs-both cases that no single-source rule can target (\cref{ssec:coalitions}). The
new science is that the guarantee survives adaptive acquisition, and that groups, not
sources, are sometimes the right unit to buy.

\section{Coalition look-ahead at $M{=}5$}\label{app:mfive}

To test the look-ahead beyond three sources we use five separately acquirable
sensor streams of the MHEALTH rig (chest accelerometer, ECG, ankle accelerometer,
ankle gyroscope, arm accelerometer; $M{=}5$, 12 activities, subject-independent
split) and rerun the residual-set resolution race from a cold start, five training runs.

\paragraph{The race.}
At $M{=}5$ the needs-a-pair fraction is $0.37\pm0.03$, comparable to the $0.28$ on the
three-source benchmarks, so coalitions do not become irrelevant as sources are added.
Cost-aware one-at-a-time policies stop early (mean spend $2.6$) and reach at most
$0.726$. Letting a single-source policy keep buying to the four-acquisition budget
lifts it to $0.786$, yet the targeted pair rescue still resolves more,
$0.850\pm0.014$ against at most $0.786\pm0.023$, at modestly higher cost ($3.19$
against $2.80$ acquisitions). The look-ahead costs $5+\binom{5}{2}=15$ candidate
evaluations per round against $31$ nonempty subsets exhaustively, at $1.3$\,ms per
case on one GPU, matching the complexity claim of \cref{sec:method}.
\Cref{tab:e2e} reports every policy end to end on the three-source benchmarks.
Under the end-to-end guarantee at $M{=}5$ (\cref{prop:rcps}, $\alpha{=}\delta{=}0.1$,
five seeds $\times$ $50$ resplits) the coalition guarantees the most cases under
either target, reference-relative $0.535$ (risk $0.004$) against at most $0.466$, and
true-label $0.467$ (risk $0.010$) against at most $0.435$, guaranteed answers being
${\approx}95\%$ correct because the reference is strong.

\begin{table*}[htbp]
\centering\small
\begin{tabular}{@{}lcccccccc@{}}
\toprule
& \multicolumn{4}{c}{CMU-MOSEI} & \multicolumn{4}{c}{IEMOCAP} \\
\cmidrule(lr){2-5}\cmidrule(lr){6-9}
\textbf{Policy} & \textbf{ans.} & \textbf{risk} & \textbf{worst} & \textbf{err$_{\mathrm{tl}}$} & \textbf{ans.} & \textbf{risk} & \textbf{worst} & \textbf{err$_{\mathrm{tl}}$} \\
\midrule
one-at-a-time (cost-aware) & $0.265\pm0.05$ & $0.049$ & $0.051$ & $0.325$ & $0.132\pm0.04$ & $0.021$ & $0.022$ & $0.138$ \\
one-at-a-time (forced)     & $0.485\pm0.04$ & $0.062$ & $0.067$ & $0.434$ & $0.201\pm0.06$ & $0.027$ & $0.028$ & $0.194$ \\
two-step sequential planner & $0.451\pm0.04$ & $0.062$ & $0.063$ & $0.412$ & $0.279\pm0.06$ & $0.031$ & $0.032$ & $0.200$ \\
open-loop pair commit      & $0.586\pm0.05$ & $0.037$ & $0.055$ & $0.461$ & $0.337\pm0.08$ & $0.015$ & $0.023$ & $0.390$ \\
\textbf{RouteCert coalition}    & $\mathbf{0.588\pm0.05}$ & $0.035$ & $0.054$ & $0.458$ & $\mathbf{0.324\pm0.03}$ & $0.014$ & $0.023$ & $0.378$ \\
\bottomrule
\end{tabular}
\caption{End-to-end guarantee under adaptive acquisition: each deployable policy
runs to its final state, then \cref{prop:rcps} ($\alpha{=}\delta{=}0.1$) applies
per terminal pattern (five seeds $\times$ $50$ resplits). Columns: guaranteed answer
rate, realized and worst per-pattern reference-relative risk, true-label error
among answered cases (``err$_{\mathrm{tl}}$'', context only). All risk $\le0.10$
under the system's own routing; the coalition adds $10$--$12$ answer points over
the best single source, beats the two-step planner ($0.588$ vs $0.451$ on
CMU-MOSEI), and matches the open-loop pair commit: replanning is the deficit.}
\label{tab:e2e}
\end{table*}

\section{Sensitivity, robustness, and ablations}\label{app:sens}

\paragraph{Family-wise control, measured.} \Cref{prop:rcps} states each pattern's
guarantee at $\delta{=}0.1$; simultaneous control of all $|\mathcal{S}|$ patterns
follows by running the identical rule at the Bonferroni level $\delta/|\mathcal{S}|$.
Running both arms on all three risk-control benchmarks: on CMU-MOSEI
($\delta/5{=}0.02$) the simultaneous guarantee holds every pattern at worst risk
$0.023$ for $5.6$ answer-rate points ($0.353\to0.297$); on IEMOCAP, worst $0.000$ at
$0.217\to0.188$. On the smallest pool (MHEALTH) the cost is near-total abstention
(MHEALTH, $|\mathcal{S}|{=}7$, $0.335\to0.010$): the tighter
Clopper--Pearson quantile needs more error-free validation cases than the fragmented
pools provide, which is exactly the floor of \cref{rem:cost} binding. Simultaneous
control is thus practical on the larger benchmarks; headline numbers are per-pattern
$\delta{=}0.1$ and labeled as such. The same accounting bounds scaling: terminal patterns
can number up to $2^M$, so rare patterns increasingly hit the abstain-entirely branch,
but a provable fallback keeps the method usable as $M$ grows.

\Cref{tab:simulfull} runs all of the simultaneous arms of \cref{rem:cost} in a
single matched pass, so the comparison is free of run-to-run noise: one training
run per benchmark, one set of $50$ resplits, one cutoff-selection rule shared by
the test-corrected arms. Two levers appear. Correcting the \emph{test} only,
comparing each $p_\Sset$ to $\delta/|\mathcal{S}|$, is the baseline the
domination argument targets, and Holm and Hochberg meet or beat it on every
benchmark as proved ($+0.009$ answered on CMU-MOSEI, $+0.005$ on MHEALTH for
Hochberg, ties elsewhere). Correcting the \emph{selection} of the candidate
cutoff as well, the variant reported in the paragraph above, moves the answer
rate independently and in either direction: it gains $0.023$ on CMU-MOSEI, where
the more conservative candidate clears its test more easily, and loses $0.007$ on
MHEALTH, where thin pools leave little certifiable. The two levers compose,
and only the first carries a domination guarantee.

\begin{table*}[htbp]
\centering\small
\setlength{\tabcolsep}{4pt}
\begin{tabular}{@{}lcccccc@{}}
\toprule
& \textbf{pointwise} & \multicolumn{4}{c}{\textbf{simultaneous at $90\%$}} & \textbf{global} \\
\cmidrule(lr){2-2}\cmidrule(lr){3-6}\cmidrule(lr){7-7}
\textbf{Benchmark} & \textbf{per-pat.} & \textbf{Bonf.} & \textbf{Holm} & \textbf{Hochb.} & \textbf{Bonf.$^\ast$} & \textbf{recal.} \\
\midrule
CMU-MOSEI & $0.350$ & $0.274$ & $0.283$ & $0.283$ & $0.305$ & $0.342$ \\
IEMOCAP   & $0.203$ & $0.188$ & $0.188$ & $0.188$ & $0.188$ & $0.189$ \\
MHEALTH   & $0.294$ & $0.007$ & $0.007$ & $0.012$ & $0.007$ & $0.341$ \\
\midrule
\multicolumn{7}{@{}l}{\emph{worst-pattern realized selective risk (target $0.10$)}} \\
CMU-MOSEI & $.034$ & $.018$ & $.021$ & $.021$ & $.024$ & $.145$ \\
IEMOCAP   & $.004$ & $.000$ & $.000$ & $.000$ & $.000$ & $.064$ \\
MHEALTH   & $.014$ & $.001$ & $.001$ & $.002$ & $.001$ & $.092$ \\
\bottomrule
\end{tabular}
\caption{Simultaneous risk control, one matched run per benchmark ($50$
resplits, $\alpha{=}\delta{=}0.1$). The MHEALTH row is the fixed-sequence protocol
run over $G{=}7$ patterns, so its pointwise and global columns are not the run behind
the MHEALTH row of \cref{tab:pervsglobal} and the two are not directly comparable. ``Bonf.'', Holm and Hochberg correct the
validation test only and share a cutoff-selection rule, which is the setting of
\cref{rem:cost}; the step procedures never fall below Bonferroni there.
``Bonf.$^\ast$'' also tightens selection to $\delta/|\mathcal{S}|$. Global
recalibration is shown for scale: it answers most but is the only column whose
worst-pattern risk breaches the cap. A fixed-sequence arm ordered by pool size
is also valid but far weaker ($0.146$/$0.068$/$0.029$), since with
comparably sized patterns the ordering predicts little and the sequence
truncates after $0.1$--$0.9$ patterns.}
\label{tab:simulfull}
\end{table*}

\paragraph{Scalable calibration by pattern coarsening.} \Cref{asm:polcal} is
agnostic to which threshold-free partition indexes the groups: any function of the
augmented pair measurable before calibration yields a valid Mondrian guarantee. So
replace the fine final-pattern partition (up to $2^M$ cells) by a coarser
threshold-free one, e.g.\ grouping terminal patterns by the number of missing sources
($M{+}1$ cells). Since the coarsening depends only on the routing output, the
augmented-exchangeability argument of \cref{app:theory} applies verbatim at the
coarser resolution, so \cref{thm:coverage,prop:rcps} hold with the coarse cell in
place of $\Sset$.

\begin{proposition}[Coverage under a coarser routing partition]\label{prop:coarsen}
Let $\Pi$ be any partition of the terminal patterns that is a measurable function of
the threshold-free routing output (independent of the calibration scores). Then
under \cref{asm:polcal}, for each cell $B\in\Pi$, the calibrated guarantee covers
its target at level $1-\alpha$ and the risk rule of \cref{prop:rcps} controls
selective risk at $\alpha$, both conditional on $B$.
\end{proposition}

\begin{corollary}[Data-adaptive coarsening]\label{cor:adaptcoarsen}
Split the calibration set into a selection part and a disjoint
guarantee part. Let $B(\cdot)$ assign each terminal pattern a cell
label from a fixed threshold-free hierarchy (for instance exact pattern
$\to$ missing-count cell $\to$ root) by any rule measurable with respect to the
selection part, such as the finest cell whose routed selection count is at least
$m_0$, backing off when it is not. The deployed groups are the preimages
$\{\Sset: B(\Sset)=b\}$, one per assigned label $b$; these are disjoint by
construction even when different patterns back off to different levels of the
hierarchy (a pattern that keeps its exact cell is thereby excluded from
any coarser cell's group, so no guarantee pool overlaps another).
Guaranteeing within each preimage group on the guarantee part,
\cref{thm:coverage} and \cref{prop:rcps} hold conditional on the assigned group;
for \cref{prop:rcps} the guarantee part of each group is itself split again
into its selection and validation halves, as always.
\end{corollary}

\begin{proof}
$B$ is a fixed function of the selection part, hence independent of the
guarantee scores (the two calibration parts are disjoint), and the preimages
$\{\Sset:B(\Sset)=b\}$ partition the
terminal patterns. Conditioning on $B$, this partition is a fixed threshold-free
coarsening of the final-pattern partition, so \cref{prop:coarsen} applies to the
guarantee part and gives coverage and the risk control of \cref{prop:rcps}
conditional on each group.
\end{proof}

\noindent
This trades granularity for data: a data-starved pattern is guaranteed at its coarse
cell's resolution rather than abstaining, and the number of cells can be held at
$O(M)$ rather than $O(2^M)$, so calibration size per cell scales as $n/O(M)$,
avoiding exponential fragmentation. The practitioner chooses the partition, from
fully per-pattern (finest, most abstention) to by-missing-count ($M{+}1$ cells),
before seeing calibration; the backoff that guarantees each case at the finest cell
its data supports, coarsening only when data-starved, is the data-adaptive instance
\cref{cor:adaptcoarsen} makes valid, with the resolution chosen on a held-out
selection split.

Scope. The coarsening guarantee is exercised at $M{=}5$, at $M{=}8$, and on a second
wearable benchmark; the results establish that the guarantee, not merely
the heuristic, extends to coarser partitions. The escape is
sharpest at $M{=}5$ under simultaneous control, where exact per-pattern conditioning answers
$0.51$ pointwise but nearly collapses to $0.10$ once the Bonferroni split runs over
${\approx}26$ cells, while by-missing-count coarsening ($O(M)$ cells) answers $0.87$
pointwise and retains $0.61$ under simultaneous control (a sixfold recovery), every
resolution holding realized risk at or below the cap.

The same escape holds at a higher source count and under direct true-label
calibration on the two strong-reference wearable benchmarks.
Tokenizing MHEALTH into eight per-sensor streams ($M{=}8$) fragments the exact
per-pattern partition into ${\sim}160$ observed terminal patterns, at which exact
conditioning abstains entirely; grouping by missing-count restores a true-label
operating point of $0.74$ answered at $0.4\%$ realized risk. This confirms that the $O(M)$ partition of \cref{prop:coarsen} keeps the true-label guarantee usable as the source and pattern counts grow.

\subsection{Robustness and ablations}\label{app:ablation}
Three checks support the main results. The optional true-label screen retargets the
guarantee from the full-information model to the real label, trading answer rate for
true-label risk on CMU-MOSEI as follows.

\begin{center}\small
\begin{tabular}{@{}lcccc@{}}
\toprule
\textbf{screen setting} & 1 & 2 & 3 & 4 \\
\midrule
answered & $15.9\%$ & $5.3\%$ & $1.9\%$ & $0.6\%$ \\
true-label risk & $0.44$ & $0.32$ & $0.25$ & $0.07$ \\
\bottomrule
\end{tabular}
\end{center}

\noindent
The screen needs plentiful per-pattern true-label calibration data, so on the smaller
IEMOCAP it moves true-label risk little at usable answer rates. Perturbing the model
that samples missing sources leaves the guarantee intact, and a sweep over the number
of completions $K$ is stable once $K$ is moderate.
\begin{table*}[t]
\centering\small
\setlength{\tabcolsep}{3pt}
\begin{tabular}{@{}lccccccc@{}}
\toprule
 & \textbf{needs} & \textbf{cost} & \multicolumn{2}{c}{\textbf{guaranteed}} & &
 \textbf{true-label err.} & \\
\cmidrule(lr){4-5}
\textbf{Benchmark (ref.\ acc.)} & \textbf{a pair} & \textbf{sgl$\to$C} &
\textbf{sgl.} & \textbf{RouteCert} & $\Delta$ & \textbf{sgl.\,/\,RouteCert} & \textbf{cells} \\
\midrule
\multicolumn{8}{@{}l}{\emph{Block 1. Target: the true label. Guaranteed true-label correctness.}} \\
MHEALTH-5 ($0.97$)           & $0.37$ & $3.1{\to}3.4$ & $0.435$ & $\mathbf{0.467}$ & $+.032$ & $.011/.010$ & $25{\to}27$ \\
\midrule
\multicolumn{8}{@{}l}{\emph{Block 2. Target: $y^{\mathrm{full}}$. Guaranteed reference-decision stability;}} \\
\multicolumn{8}{@{}l}{\emph{\phantom{Block 2. }true-label error is an unguaranteed diagnostic.}} \\
CMU-MOSEI ($0.50$) & $0.28$ & $1.5{\to}1.8$ & $0.485$ & $\mathbf{0.588}$ & $+.103$ & $.434/.458$ & $6{\to}6$ \\
IEMOCAP ($0.60$)   & $0.28$ & $1.9{\to}2.1$ & $0.201$ & $\mathbf{0.324}$ & $+.123$ & $.194/.378$ & $4{\to}5$ \\
\bottomrule
\end{tabular}
\caption{Coalition against the stronger single-source policy, split by guarantee
target. Reference accuracy is in parentheses; Block~2's true-label column is a
diagnostic no guarantee covers, so a small reference-relative risk there means the
system reproduces the reference, not that either is correct. ``guaranteed'':
answer rate after RouteCert-Risk; ``sgl.'': the stronger single-source policy;
``cost'': mean sources acquired; ``cells'': terminal patterns created. Worst-pattern
risk stays under the cap throughout ($\le0.046$). $^\ddagger$$O(M)$ coarsening.}
\label{tab:maintl}
\end{table*}

\section{Clinical ECG case study (PTB-XL)}\label{app:ptbxl}
This section documents the ECG study of \cref{ssec:ptbxl} in full.

\paragraph{Task, sources and folds.}
PTB-XL v1.0.3 \citep{wagner2020ptbxl,goldberger2000physionet} provides
$12$-lead records with diagnostic superclasses assigned by a cardiologist. We keep
one record per patient (minimum \texttt{ecg\_id}) and predict the myocardial-
infarction superclass against all other records. Acquisition proceeds through three
nested stages, lead~I, leads~I+II, and the clinical eight-lead set, at cumulative
ordinal costs $0.1$, $1.0$ and $5.0$, so the full recording costs $5.0$ and the
increments are $0.1$, $0.9$ and $4.0$. The official patient-safe folds are used throughout, one role each: folds~1--5 fit the nine XResNet1d18 models (three stages $\times$ three seeds), fold~6 selects each checkpoint, fold~7 fits the absolute-logit correctness calibrators, fold~8 selects the single confidence threshold $0.90$, fold~9 ($n{=}1942$) certifies, and fold~10 ($n{=}1904$) is opened once for the held-out test. No patient appears in two folds.

\paragraph{What was fixed in advance.}
Every design choice was set before the certification and test folds were opened:
the dataset version and patient selection, the label, the three stages and their
costs, the nine fitted checkpoints and their seeds, the fold-7 calibrators, the
confidence threshold $0.90$ selected once on fold~8, $\alpha{=}0.1$, total
$\delta{=}0.1$, and Holm over exactly nine hypothesis slots. One earlier
requirement, that some acquired prefix also certify an unconditional predict-all
hypothesis, was dropped because it tests an all-case classifier while a selective
method is permitted to abstain after acquiring everything; nothing else changed,
and folds~9 and~10 were opened only afterward. This is prospective methodology
development rather than independent confirmation, which would require a second
site.

\paragraph{Held-out result.}
\Cref{tab:ptbxl} reports the nine certified slots and \cref{fig:ptbxl} the held-out
cost--answer frontier. The
valid arm certifies all three adaptive terminals and answers $71.2\%$ of fold-10
patients at $7.4\%$ disagreement, within the $0.1$ cap, at attempted ordinal cost
$2.442$ of $5.0$. The registered fixed-prefix controls do not dominate it: the
fixed-selective family certifies at every stage but spends more to reach the same
answered fraction, and no fixed-unconditional hypothesis certifies at all, since
predicting for every patient at any single stage exceeds the cap
($0.192$, $0.152$ and $0.115$ at lead~I, limb and clinical respectively).

\begin{figure*}[htbp]
\centering
\includegraphics[width=0.98\textwidth]{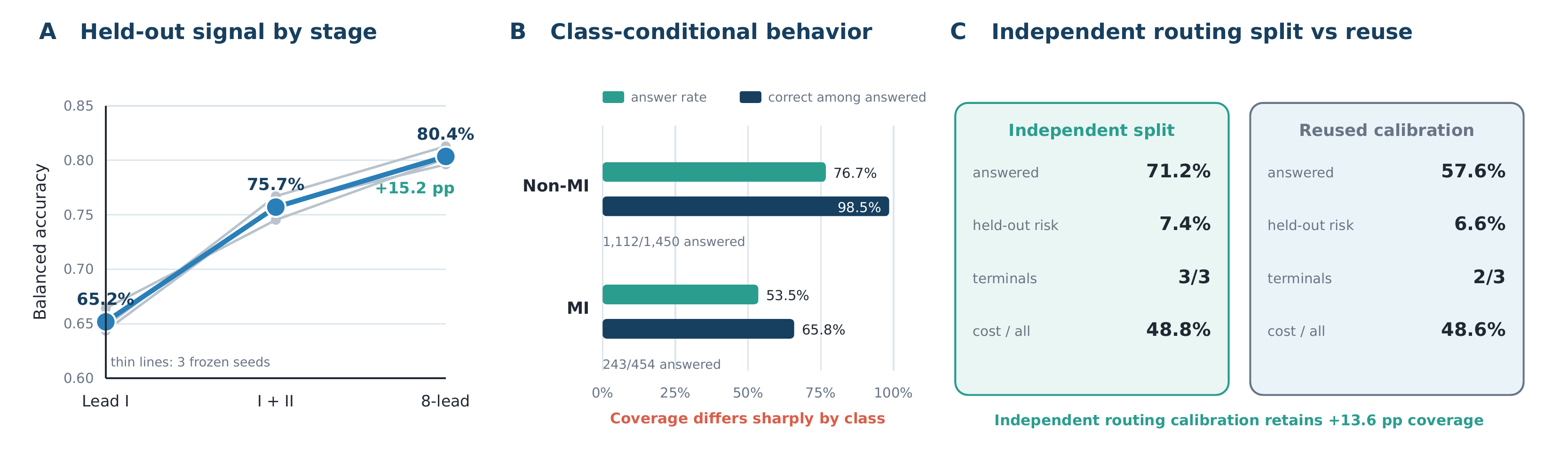}
\caption{\textbf{What the held-out PTB-XL fold adds beyond \cref{fig:ptbxlflow}.}
\textbf{(A)} Balanced accuracy by acquired stage over three frozen seeds, so each
added lead contributes held-out signal rather than only cost. \textbf{(B)} Answer
rate and correctness among answered patients by reference class, the main clinical
qualification, since abstention and residual error concentrate on the MI class.
\textbf{(C)} Certifying on a draw independent of the one routing consults answers $71.2\%$ against $57.6\%$ and certifies three terminals against two, so the independent split \cref{rem:sharp} calls for costs nothing here. The patient flow, the
cost--answer frontier and the risk at each terminal are in \cref{fig:ptbxlflow},
and the nine certified slots in \cref{tab:ptbxl}.}
\label{fig:ptbxl}
\end{figure*}

\paragraph{Tier-2 control and class-conditional behavior.}
The nominal arm lets routing consult the same draw used to certify it, which
\cref{rem:sharp} says is not licensed. It answers $57.6\%$ at $6.6\%$ risk and
certifies only two terminals, so on this dataset the invalid design is worse rather
than dangerously better, and we report it as neutral. Answering is not uniform
across classes in either arm: the valid arm answers $76.7\%$ of non-MI against
$53.5\%$ of MI patients. Abstention therefore concentrates on the positive class,
which matters clinically and is a property of the operating point rather than of the
guarantee.

\begin{table}[t]
\centering\small
\setlength{\tabcolsep}{4pt}
\begin{tabular}{@{}llrrc@{}}
\toprule
\textbf{Family} & \textbf{Terminal} & $n$ & \textbf{emp.\ risk} & \textbf{certified} \\
\midrule
adaptive          & lead~I    & $618$  & $0.049$ & yes \\
adaptive          & limb      & $452$  & $0.035$ & yes \\
adaptive          & clinical  & $318$  & $0.066$ & yes \\
fixed-selective   & lead~I    & $618$  & $0.049$ & yes \\
fixed-selective   & limb      & $955$  & $0.023$ & yes \\
fixed-selective   & clinical  & $1216$ & $0.032$ & yes \\
fixed-uncond.     & lead~I    & $1942$ & $0.192$ & no  \\
fixed-uncond.     & limb      & $1942$ & $0.152$ & no  \\
fixed-uncond.     & clinical  & $1942$ & $0.115$ & no  \\
\bottomrule
\end{tabular}
\caption{The nine Holm-corrected slots on the certification fold ($n{=}1942$),
$\alpha{=}\delta{=}0.1$. Adaptive terminals are the patients the deployed policy
stops at; fixed-selective applies the same cutoff at a fixed stage;
fixed-unconditional predicts for every patient at that stage and certifies nowhere.
Held-out fold-10 deployment of the certified adaptive policy answers $0.712$ at
risk $0.074$ and cost fraction $0.488$.}
\label{tab:ptbxl}
\end{table}

\paragraph{Limitations.}
Five apply. The certified quantity is agreement with a cardiologist's assigned
diagnostic superclass, not acute-MI outcome truth. All leads are recorded
simultaneously, so acquisition is simulated by withholding leads rather than by
delaying a measurement, and the timing cost of a real staged acquisition is not
represented. Costs are ordinal burden tiers, not measured time or money. The
protocol was fixed before the certification and test folds were opened, so this is a
development result rather than independent confirmation. And it is a single-site cohort with no
external-site validation.

\section{Setup and reproducibility}\label{app:setup}

\paragraph{Code.}
A code and data package containing the implementation, job scripts, cached outputs
where available, and a manifest mapping paper objects to artifacts is available from
the author on request. It independently verifies the PTB-XL certificate from shipped
logits. Raw datasets, trained general-benchmark checkpoints, and several supplementary
cached outputs are not included, so the package does not claim exact regeneration of
every reported number.

\paragraph{Splits and models.} On each benchmark the held-out test pool is split
$50/50$ into calibration and evaluation halves, and per-pattern statistics are
averaged over $50$ such splits, which share one trained model per seed and so
quantify calibration randomness only. Test pools hold $4{,}643$ (CMU-MOSEI) and
$938$ (IEMOCAP) cases. The reference is a transformer fusion model (d\_model $256$,
$8$ heads, $3$ layers, $120$ epochs) with validation accuracy $0.51$ on CMU-MOSEI
and $0.71$ on IEMOCAP, while held-out accuracy against ground truth is lower,
$0.49$--$0.50$ and $0.58$--$0.63$ (mean $0.60$) across five seeds. On the wearable benchmarks we quote the reference's validation
accuracy, $0.96$ on the three-pack MHEALTH used for risk control and
$0.96$--$0.97$ on the five-sensor variant (\cref{app:mfive}). The completion model is a masked conditional VAE per benchmark:
an encoder over the observed blocks produces a shared latent (dimension $64$, hidden
width $256$), per-source decoders reconstruct the missing blocks, and training
minimizes masked reconstruction error plus a $\beta{=}0.1$ KL term ($80$ epochs)
under random per-source missingness ($0.4$ per source; at least one source kept on
$85\%$ of batches, the all-missing cold-start mask on the remaining $15\%$). All
missing blocks of one draw are decoded from a single latent sample, so cross-source
dependence is carried by the latent.

\paragraph{Missing-pattern generation.} Each experiment fixes its initial
patterns before any method runs, with no example duplicated across patterns.
The risk-control, adaptive-validity, and mixture-shift studies assign every
held-out case one missing pattern uniformly at random over the benchmark's
pattern set (on CMU-MOSEI: none, audio, vision, text, audio$+$vision missing).
The error-versus-coverage studies start each case with exactly one observed
source chosen uniformly at random. The resolution races start cold, all
sources missing. Calibration cases are pushed through the identical loop, so
each terminal pattern's calibration split is produced by the same routing as its
test split.

\paragraph{Calibration-pool requirements.} The guarantee assumes a calibration pool
in which every acquirable source is present (and carries true labels for the
optional true-label layer), so each terminal pattern can be scored against completions
and, when needed, ground truth. Our controlled-missingness protocol satisfies this
by construction: we start from fully observed held-out cases and hide sources, so
the hidden blocks are always available at calibration and acquisition. A
natural-missingness deployment may not offer such a pool; there the calibration set
must be assembled from cases that have the sources present, which we flag as a data
requirement rather than a free property. Testing that regime properly needs a
cohort in which acquisition was actually ordered and timestamped. A linked emergency-department study such as MIMIC-IV is the natural candidate, and it is future work rather than something the present benchmarks can stand in for.

To probe this regime we induce initial missingness through a label-correlated MNAR
mechanism of strength $\gamma$ (each case's missing source is drawn from a
distribution the label pulls on; $\gamma{=}0$ is missing-completely-at-random) and
calibrate the per-pattern coverage guarantee either through the same mechanism
(matched) or from complete cases masked uniformly (complete-case),
the latter conditioning on the wrong within-pattern label law. The complete-case
gap is small: on CMU-MOSEI its worst per-pattern coverage falls from $0.89$ to
$0.86$ as $\gamma$ grows, about a point more than matched, because the
nonconformity score is only weakly label-dependent; on IEMOCAP the gap stays
within finite-sample noise. Reweighting the complete-case
calibration by the modeled inclusion odds restores coverage. That gap is exactly
the shift term a natural-missingness deployment would add to
$r_{\mathrm{o}}+\bar\varepsilon_{\mathrm{ref}}$: here it is small and estimable
when the mechanism is known.

\paragraph{The deployable acquisition-score estimator.} For a candidate group
$\acoal$ at state $(x_\Sset,\Sset)$, the look-ahead score is
$\Delta(\acoal)=\bigl(|\Lset(x_\Sset)|-|\Lset(x_{\Sset\cup\acoal})|\bigr)/c(\acoal)$,
the per-cost reduction in residual-set size, both sizes from $K$ completion draws.
This is a cost-normalized value-of-information criterion in the spirit of active
feature acquisition \citep{ma2019eddi,li2021gsmrl}; what is new is the unit, a group
rather than a single source, and that it reads only $\Lset$ and the models, never a
calibrated threshold, which keeps Phase~1 valid for the guarantee. In our
simulations the candidate's own block is revealed at its true held-out value when
scoring (all policies alike, so comparisons are unaffected). A deployed system
instead averages $\Delta$ over \gen-draws before paying, and
reruns the resolution race with that estimator (five \gen-draws per candidate, true
block revealed only after purchase): every policy resolves less, and the coalition
gain narrows but persists ($+7.1{\pm}0.4$ points over the forced single-source
control on CMU-MOSEI, $+10.5{\pm}1.6$ on IEMOCAP, versus $+16.3$ and $+23.4$
revealed), positive for every seed.

\begin{algorithm}[t]
\caption{Deployable coalition scoring at one acquisition step (nested Monte Carlo)}
\label{alg:deploy}
\begin{algorithmic}[1]
\Require observed $x_\Sset$, pattern $\Sset$, candidate groups $\{\acoal\}$ of
size $\le r$, completion model \gen, reference model \rcls, costs $\{c(\acoal)\}$, outer
draws $J$, inner fill-ins $K$
\State $\ell_0 \gets |\Lset(x_\Sset)|$ \Comment{residual size from $K$ inner
completions of the missing sources at the current state}
\For{each candidate group $\acoal$}
  \For{$j=1,\dots,J$} \Comment{outer draws, independent across candidates}
    \State draw a hypothetical block $\tilde{x}^{(j)}_{\acoal}\sim\gen(\cdot\mid x_\Sset)$
    \State draw $K$ fresh inner completions of the sources still missing
    after $\acoal$, from $\gen(\cdot\mid x_\Sset,\tilde{x}^{(j)}_{\acoal})$; run
    \rcls\ on each; their distinct answers form $\Lset^{(j)}$
    \State $g_j \gets \bigl(\ell_0-|\Lset^{(j)}|\bigr)/c(\acoal)$
  \EndFor
  \State $\widehat{\Delta}(\acoal)\gets \tfrac{1}{J}\sum_{j=1}^{J} g_j$
\EndFor
\State $\acoal^\star \gets \arg\max_{\acoal}\widehat{\Delta}(\acoal)$
\Comment{singles first, pair rescue only if no single scores positive, as in \cref{alg:routecertmain}}
\State pay $c(\acoal^\star)$, reveal the true blocks of $\acoal^\star$, set
$x_\Sset\gets x_\Sset\cup x_{\acoal^\star}$ and $\Sset\gets\Sset\cup\acoal^\star$
\Comment{$\Sset$ is the observed set, so acquiring adds to it}
\State recompute the state and repeat until $\Lset$ is a singleton, no group
shrinks it, or the budget is spent
\end{algorithmic}
\end{algorithm}

Each candidate costs $J$ outer draws times $K$ inner \rcls/decoder evaluations,
so about $J\cdot K$ reference-model and completion calls per candidate per step
($J{=}5$, $K{=}15$ in our runs). The $K$ inner completions are resampled afresh
inside every outer draw, a nested Monte Carlo estimate rather than a shared inner
pool, and the $J$ outer draws are independent across candidates, so no randomness
is reused between two candidates being compared. The revealed-block estimator of
the paragraph above is the $J{=}1$ special case with $\tilde{x}_{\acoal}$ pinned
to the candidate's true held-out block, which is why it is cheaper and why all
policies share it in the like-for-like ablation.

\subsection{Reproducibility}\label{app:repro}
All reported numbers come from batch jobs on a compute cluster, and every per-job configuration is in the code package. Each baseline is a matched score-level reimplementation on the shared backbone: it controls for the backbone and isolates the acquisition score, but it is not a reproduction of the published system, and no end-to-end superiority claim rests on it. The controlled comparison freezes each policy's route and varies only the certificate (\cref{tab:confbase}). Full proofs of \cref{thm:coverage,thm:selrisk,prop:rcps} and \cref{prop:synergy} appear in \cref{app:theory}; all protocols and per-job configurations are in the code package.

\bibliography{ref}
\end{document}